\documentclass{article}

\usepackage[preprint]{neurips_2026}

\usepackage[utf8]{inputenc}
\usepackage[T1]{fontenc}
\usepackage{hyperref}
\usepackage{url}
\usepackage{booktabs}
\usepackage{amsfonts}
\usepackage{amsmath}
\usepackage{amssymb}
\usepackage{amsthm}
\usepackage{mathtools}
\usepackage{nicefrac}
\usepackage{microtype}
\usepackage[table]{xcolor}
\usepackage{graphicx}
\usepackage{multirow}
\usepackage{enumitem}
\usepackage[capitalize,noabbrev]{cleveref}

\theoremstyle{plain}
\newtheorem{theorem}{Theorem}[section]
\newtheorem{proposition}[theorem]{Proposition}
\newtheorem{lemma}[theorem]{Lemma}

\theoremstyle{definition}
\newtheorem{definition}[theorem]{Definition}

\theoremstyle{remark}
\newtheorem{remark}[theorem]{Remark}

\newcommand{\mrr}[2]{\ensuremath{#1 \pm #2}}

\newcommand{\na}{\textemdash}

\newcommand{\valActivityOnlyAlpha}{\mrr{0.9529}{0.0006}}
\newcommand{\valActivityOnlyMOOC}{\mrr{0.1421}{0.0013}}
\newcommand{\valAlphaGATPR}{\mrr{0.9199}{0.0183}}
\newcommand{\valAlphaGATROC}{\mrr{0.6735}{0.0345}}
\newcommand{\valAlphaGCNPR}{\mrr{0.9156}{0.0062}}
\newcommand{\valAlphaGCNROC}{\mrr{0.6666}{0.0137}}
\newcommand{\valAlphaMotifsSAGEPR}{\mrr{0.9667}{0.0008}}
\newcommand{\valAlphaMotifsSAGEROC}{\mrr{0.8302}{0.0029}}
\newcommand{\valAlphaMotifsTGNPR}{\mrr{0.9703}{0.0025}}
\newcommand{\valAlphaMotifsTGNROC}{\mrr{0.8440}{0.0134}}
\newcommand{\valAlphaSAGEPR}{\mrr{0.9282}{0.0055}}
\newcommand{\valAlphaSAGEROC}{\mrr{0.6990}{0.0069}}
\newcommand{\valAlphaTGNPR}{\mrr{0.9312}{0.0076}}
\newcommand{\valAlphaTGNROC}{\mrr{0.7097}{0.0182}}
\newcommand{\valFamAlphaAAll}{\mrr{0.9667}{0.0008}}
\newcommand{\valFamAlphaAOne}{\mrr{0.9699}{0.0002}}
\newcommand{\valFamAlphaAOneAthree}{\mrr{0.9699}{0.0004}}
\newcommand{\valFamAlphaAOneAtwo}{\mrr{0.9674}{0.0009}}
\newcommand{\valFamAlphaAThree}{\mrr{0.9287}{0.0056}}
\newcommand{\valFamAlphaATwo}{\mrr{0.9503}{0.0015}}
\newcommand{\valFamAlphaAtwoAthree}{\mrr{0.9491}{0.0009}}
\newcommand{\valFamMOOCAAll}{\mrr{0.1639}{0.0028}}
\newcommand{\valFamMOOCAOne}{\mrr{0.0301}{0.0003}}
\newcommand{\valFamMOOCAOneAthree}{\mrr{0.0451}{0.0010}}
\newcommand{\valFamMOOCAOneAtwo}{\mrr{0.1644}{0.0007}}
\newcommand{\valFamMOOCAThree}{\mrr{0.0273}{0.0002}}
\newcommand{\valFamMOOCATwo}{\mrr{0.1430}{0.0003}}
\newcommand{\valFamMOOCAtwoAthree}{\mrr{0.1439}{0.0007}}
\newcommand{\valFamPaySimAAll}{\mrr{0.4430}{0.0589}}
\newcommand{\valFamPaySimAOne}{\mrr{0.5062}{0.1695}}
\newcommand{\valFamPaySimAOneAthree}{\mrr{0.5062}{0.1695}}
\newcommand{\valFamPaySimAOneAtwo}{\mrr{0.4430}{0.0589}}
\newcommand{\valFamPaySimAThree}{\mrr{0.3756}{0.0887}}
\newcommand{\valFamPaySimATwo}{\mrr{0.4321}{0.1552}}
\newcommand{\valFamPaySimAtwoAthree}{\mrr{0.4321}{0.1552}}
\newcommand{\valGCGatAcc}{\mrr{0.1667}{0.0000}}
\newcommand{\valGCGatFOne}{\mrr{0.0476}{0.0000}}
\newcommand{\valGCGatMotAcc}{\mrr{0.9400}{0.0149}}
\newcommand{\valGCGatMotFOne}{\mrr{0.9398}{0.0147}}
\newcommand{\valGCGcnAcc}{\mrr{0.7333}{0.0204}}
\newcommand{\valGCGcnFOne}{\mrr{0.6774}{0.0250}}
\newcommand{\valGCGcnMotAcc}{\mrr{0.9267}{0.0091}}
\newcommand{\valGCGcnMotFOne}{\mrr{0.9261}{0.0094}}
\newcommand{\valGCRfAcc}{0.9333}
\newcommand{\valGCRfFOne}{0.9327}
\newcommand{\valGCSageAcc}{\mrr{0.5933}{0.0652}}
\newcommand{\valGCSageFOne}{\mrr{0.5392}{0.0980}}
\newcommand{\valGCSageMotAcc}{\mrr{0.9167}{0.0167}}
\newcommand{\valGCSageMotFOne}{\mrr{0.9145}{0.0186}}
\newcommand{\valGCSvmAcc}{0.8000}
\newcommand{\valGCSvmFOne}{0.7948}
\newcommand{\valLooAlphaA}{-0.01\%}
\newcommand{\valLooAlphaB}{-0.06\%}
\newcommand{\valLooAlphaC}{-0.07\%}
\newcommand{\valLooAlphaD}{-0.07\%}
\newcommand{\valLooAlphaE}{+0.00\%}
\newcommand{\valLooAlphaF}{+0.10\%}
\newcommand{\valLooAlphaG}{-0.29\%}
\newcommand{\valLooAlphaH}{+0.15\%}
\newcommand{\valLooAlphaI}{+0.18\%}
\newcommand{\valLooAlphaJ}{+0.04\%}
\newcommand{\valLooAlphaK}{-0.02\%}
\newcommand{\valLooAlphaL}{-0.04\%}
\newcommand{\valLooAlphaM}{+0.01\%}
\newcommand{\valLooMOOCA}{-0.47\%}
\newcommand{\valLooMOOCB}{-0.80\%}
\newcommand{\valLooMOOCC}{-1.84\%}
\newcommand{\valLooMOOCD}{-1.17\%}
\newcommand{\valLooMOOCE}{-1.17\%}
\newcommand{\valLooMOOCF}{-69.36\%}
\newcommand{\valLooMOOCG}{-1.23\%}
\newcommand{\valLooMOOCH}{-0.18\%}
\newcommand{\valLooMOOCI}{+0.38\%}
\newcommand{\valLooMOOCJ}{-3.86\%}
\newcommand{\valLooMOOCK}{-0.36\%}
\newcommand{\valLooMOOCL}{-0.34\%}
\newcommand{\valLooMOOCM}{-0.53\%}
\newcommand{\valLooPaySimA}{+3.56\%}
\newcommand{\valLooPaySimB}{+0.00\%}
\newcommand{\valLooPaySimC}{+13.64\%}
\newcommand{\valLooPaySimD}{+0.00\%}
\newcommand{\valLooPaySimE}{+0.00\%}
\newcommand{\valLooPaySimF}{-3.68\%}
\newcommand{\valLooPaySimG}{+0.57\%}
\newcommand{\valLooPaySimH}{-19.40\%}
\newcommand{\valLooPaySimI}{+10.76\%}
\newcommand{\valLooPaySimJ}{+0.00\%}
\newcommand{\valLooPaySimK}{+0.00\%}
\newcommand{\valLooPaySimL}{+0.00\%}
\newcommand{\valLooPaySimM}{+0.00\%}
\newcommand{\valMOOCGATPR}{\mrr{0.0156}{0.0002}}
\newcommand{\valMOOCGATROC}{\mrr{0.6348}{0.0024}}
\newcommand{\valMOOCGCNPR}{\mrr{0.0164}{0.0001}}
\newcommand{\valMOOCGCNROC}{\mrr{0.6347}{0.0022}}
\newcommand{\valMOOCMotifsSAGEPR}{\mrr{0.1659}{0.0005}}
\newcommand{\valMOOCMotifsSAGEROC}{\mrr{0.8371}{0.0021}}
\newcommand{\valMOOCMotifsTGNPR}{\mrr{0.1702}{0.0048}}
\newcommand{\valMOOCMotifsTGNROC}{\mrr{0.8407}{0.0019}}
\newcommand{\valMOOCSAGEPR}{\mrr{0.0164}{0.0001}}
\newcommand{\valMOOCSAGEROC}{\mrr{0.6425}{0.0024}}
\newcommand{\valMOOCTGNPR}{\mrr{0.0375}{0.0021}}
\newcommand{\valMOOCTGNROC}{\mrr{0.7340}{0.0031}}
\newcommand{\valMotifBLAlphaGlobal}{\mrr{0.9303}{0.0047}}
\newcommand{\valMotifBLAlphaStatic}{\mrr{0.9746}{0.0019}}
\newcommand{\valMotifBLAlphaTemporal}{\mrr{0.9667}{0.0008}}
\newcommand{\valMotifBLMOOCGlobal}{\mrr{0.0161}{0.0004}}
\newcommand{\valMotifBLMOOCStatic}{\mrr{0.1564}{0.0008}}
\newcommand{\valMotifBLMOOCTemporal}{\mrr{0.1659}{0.0005}}
\newcommand{\valMotifBLPaySimGlobal}{\mrr{0.4958}{0.0733}}
\newcommand{\valMotifBLPaySimStatic}{\mrr{0.3256}{0.0926}}
\newcommand{\valMotifBLPaySimTemporal}{\mrr{0.4430}{0.0589}}
\newcommand{\valMotifOnlyAlpha}{\mrr{0.9529}{0.0044}}
\newcommand{\valMotifOnlyMOOC}{\mrr{0.1172}{0.0008}}
\newcommand{\valMotifOnlyPaySim}{\mrr{0.0036}{0.0000}}
\newcommand{\valOTCGATPR}{\mrr{0.9411}{0.0004}}
\newcommand{\valOTCGATROC}{\mrr{0.7908}{0.0017}}
\newcommand{\valOTCGCNPR}{\mrr{0.9257}{0.0013}}
\newcommand{\valOTCGCNROC}{\mrr{0.7588}{0.0031}}
\newcommand{\valOTCMotifsSAGEPR}{\mrr{0.9796}{0.0008}}
\newcommand{\valOTCMotifsSAGEROC}{\mrr{0.9122}{0.0025}}
\newcommand{\valOTCMotifsTGNPR}{\mrr{0.9693}{0.0010}}
\newcommand{\valOTCMotifsTGNROC}{\mrr{0.8854}{0.0044}}
\newcommand{\valOTCSAGEPR}{\mrr{0.9288}{0.0039}}
\newcommand{\valOTCSAGEROC}{\mrr{0.7687}{0.0029}}
\newcommand{\valOTCTGNPR}{\mrr{0.9511}{0.0019}}
\newcommand{\valOTCTGNROC}{\mrr{0.8236}{0.0017}}
\newcommand{\valOverheadFairReview}{11098923.1}
\newcommand{\valOverheadOtcCand}{48817.0}
\newcommand{\valOverheadOtcCandBatch}{54.5}
\newcommand{\valOverheadOtcEpochBase}{0.0590}
\newcommand{\valOverheadOtcEpochMot}{0.0525}
\newcommand{\valOverheadOtcMem}{0.68}
\newcommand{\valOverheadOtcMemBase}{0.67}
\newcommand{\valOverheadOtcPreproc}{2.05}
\newcommand{\valOverheadRevCand}{46.0}
\newcommand{\valOverheadRevMem}{3.19}
\newcommand{\valOverheadRevPreproc}{237.26}
\newcommand{\valPaySimGATPR}{\mrr{0.2917}{0.1888}}
\newcommand{\valPaySimGATROC}{\mrr{0.9108}{0.0278}}
\newcommand{\valPaySimGCNPR}{\mrr{0.2996}{0.2116}}
\newcommand{\valPaySimGCNROC}{\mrr{0.7350}{0.1968}}
\newcommand{\valPaySimMotifsSAGEPR}{\mrr{0.4430}{0.0589}}
\newcommand{\valPaySimMotifsSAGEROC}{\mrr{0.8707}{0.0294}}
\newcommand{\valPaySimSAGEPR}{\mrr{0.3323}{0.2351}}
\newcommand{\valPaySimSAGEROC}{\mrr{0.7634}{0.1412}}
\newcommand{\valRandomFeatAlpha}{\mrr{0.9311}{0.0013}}
\newcommand{\valRandomFeatMOOC}{\mrr{0.0172}{0.0003}}
\newcommand{\valRandomFeatPaySim}{\mrr{0.4575}{0.1092}}
\newcommand{\valRecencyOnlyAlpha}{\mrr{0.9672}{0.0009}}
\newcommand{\valRecencyOnlyMOOC}{\mrr{0.0162}{0.0002}}
\newcommand{\valTgbLeaderboardCoin}{\mrr{0.832}{0.001}}
\newcommand{\valTgbLeaderboardReview}{\mrr{0.377}{0.010}}
\newcommand{\valTgbLeaderboardWiki}{\mrr{0.827}{0.001}}

\title{Temporal Motif Signatures for\\Temporal Graph Neural Networks}

\author{%
  Dylan Sandfelder \\
  University of Oxford \\
  \texttt{dylan.sandfelder@eng.ox.ac.uk}
  \And
  Mihai Cucuringu \\
  University of California, Los Angeles \\
  \texttt{mihai@math.ucla.edu}
  \And
  Xiaowen Dong \\
  University of Oxford \\
  \texttt{xdong@robots.ox.ac.uk}
}

\begin{document}

\maketitle

\begin{abstract}
  Real temporal interaction streams carry predictive structure in short-horizon motif patterns---repetition, reciprocity, star diversity, triadic flow---that vanilla temporal graph neural networks (TGNNs) often fail to expose to their edge scorers. We show this concretely on MOOC interaction prediction, where a small four-feature family of past-window star counts already delivers most of the lift over a strong static GNN. Across a wide set of real and synthetic temporal datasets we find that motif activity organizes consistently along three scale-stable axes (dyadic recency/reciprocity, star diversity, triadic flow), and we use this empirical structure to design a compact $13$-coordinate, leakage-safe, candidate-local motif feature map $h(u,v,t)$ that linearly embeds into any static or temporal encoder without architectural changes. A temporal Weisfeiler-Leman (WL) analysis places the augmentation relative to the first level of an anchored temporal-WL hierarchy and exhibits a candidate-anchored pair on which motif features distinguish. We demonstrate empirically that the same augmentation consistently lifts performance across heterogeneous tasks: TGB link-property prediction across all five baselines, edge classification on Bitcoin Alpha/OTC and MOOC, and graph-level classification of synthetic temporal generators.
\end{abstract}

\section{Introduction}
\label{sec:intro}
\vspace{-0.5em}

A continuous-time interaction stream carries information not only about who interacts with whom, but in the \emph{temporal microstructure} of interactions: whether a conversation is reciprocated within seconds or hours, whether a transaction burst involves a broadening or narrowing neighborhood, whether the third party bridging two accounts is the same today as yesterday. These phenomena are signatures of the generating process, and different domains---social communication, financial transfers, online reviews, peer lending---produce strikingly different signatures \citep{paranjape2017motifs,benson2016higher,longa2023graph}.

Modern temporal graph neural networks (TGNNs) are expressive sequence models over these interaction streams \citep{rossi2020temporal,xu2020inductive,wang2021inductive,yu2023towards,cong2023graphmixer}. They can, in principle, extract any function of the temporal stream; in practice, their inductive biases and limited training budget makes them incapable of capturing specific short-horizon structural regularities that are (a) straightforward to compute analytically, and (b) key to the specific task at hand. In this paper, we ask: \emph{what is a compact set of short-horizon statistics that characterizes a temporal stream, how do current TGNNs fare on them, and can the residual gap between the two be reduced with a principled, architecture-agnostic augmentation?}

\vspace{-0.5em}
\paragraph{Three contributions.}
Towards answering the question above, this paper makes three contributions. \textit{Contribution 1 (diagnostics, Section~\ref{sec:characterization}).} We articulate that per-candidate, past-window motif counts organize real and synthetic temporal streams along three empirically dominant axes: dyadic recency/reciprocity ($A_1$), star diversity ($A_2$), and triadic flow ($A_3$); the resulting per-dataset signatures are scale-stable and predict the per-dataset motif-induced gain. \textit{Contribution 2 (temporal expressivity, Section~\ref{sec:theory}).} We show that standard temporal message-passing GNNs (MP-GNNs) are bounded by the first level of a candidate-anchored temporal Weisfeiler-Leman hierarchy under a common edge-scoring abstraction; we exhibit a candidate-anchored pair on which motif features distinguish but temporal-$1$-WL collapses. The role of this theoretical result is to locate the augmentation relative to a standard temporal MP-GNN. \textit{Contribution 3 (empirical performance, Section~\ref{sec:experiments}).} We demonstrate that motif augmentation, compared to strong TGNN baselines, produces a reproducible \emph{motif-induced gain} on a fixed baseline$\times$dataset pair. The gain is also interpretable: clear on \texttt{tgbl-review-v2} and on \texttt{tgbl-wiki-v2} with complex temporal streams, and within one standard deviation of zero where the baseline captures the relevant short-horizon structure. 

The paper is organized as follows. Section~\ref{sec:background} presents the necessary technical background; Section~\ref{sec:toy} opens with a worked MOOC example; Section~\ref{sec:characterization} extends that analysis to a wide range of datasets; Section~\ref{sec:method} formalizes a leakage-safe per-candidate $13$-coordinate motif feature map $h(u,v,t)$ with a single linear interface to any static or temporal encoder; Section~\ref{sec:theory} gives the temporal-WL placement of this augmentation; Section~\ref{sec:experiments} measures the impact of motif augmentation across TGB link prediction, edge classification, and graph-level classification tasks; Section~\ref{sec:related} identifies the gaps in related work; Section~\ref{sec:limitations} discusses the limitations of this framework.

\vspace{-0.5em}
\section{Background}
\label{sec:background}
\vspace{-0.5em}

\paragraph{Temporal interaction streams.}
A directed temporal interaction stream is a sequence $\mathcal{D}=\{(u_i,v_i,t_i,x_i,y_i)\}_{i=1}^{N}$ where $u_i,v_i\in V$, $t_i\in\mathbb{R}$, $t_1\le\cdots\le t_N$, $x_i$ is an edge-local feature vector, and $y_i$ is a (possibly partial) task label. The \emph{underlying static graph} $G_{\mathrm{static}}=(V,E)$ has $E=\{(u_i,v_i):1\le i\le N\}$. Here, multiple temporal events between a given $(u,v)$ collapse to a single static edge.

\vspace{-0.5em}
\paragraph{Message passing and WL test.}
A message-passing graph neural network (MP-GNN) updates node states via
$h_v^{(t)} = \psi^{(t)}(h_v^{(t-1)}, \mathrm{AGG}^{(t)}\{\phi^{(t)}(h_v^{(t-1)},h_u^{(t-1)},x_{uv}) : (u,v)\in E\})$
for $t=1,\dots,T$. Standard MP-GNNs are bounded by the $1$-dimensional Weisfeiler-Leman test ($1$-WL) \citep{morris2019weisfeiler,xu2018powerful}; temporal MP-GNNs that add time encodings \citep{xu2020inductive,rossi2020temporal,yu2023towards,cong2023graphmixer} lift this bound to a \emph{temporal} analogue developed in Section~\ref{sec:theory}.

\vspace{-0.5em}
\paragraph{$\delta$-temporal motifs.}
Given a small directed multigraph $H$ with $m$ edges and an ordering of those edges, a $\delta$-temporal motif instance in $\mathcal{D}$ is a sequence of $m$ temporal edges matching $H$ structurally, with timestamps strictly increasing and spanning at most $\delta$ time \citep{paranjape2017motifs}. Our motif family (Section~\ref{sec:method}) specializes this to edge-anchored, past-only counts with bounded $\delta=\Delta$. Section~\ref{sec:toy} and Section~\ref{sec:characterization} use the term \textit{``motif count in a past window''}  informally before Section~\ref{sec:method} formalizes the $13$-coordinate map $h(u,v,t)$. Extended background on MP-GNNs, WL, and temporal motif literature is provided in Appendix~\ref{app:extended_related}.

\vspace{-0.5em}
\section{A motivating example: temporal stars in MOOC}
\label{sec:toy}
\vspace{-0.5em}

To make the residual structural gap concrete, we focus on a single dataset and single feature before introducing any formal machinery.

\vspace{-0.5em}
\paragraph{Setup.}
The MOOC interaction-classification task \citep{kumar2019predicting} is a $\sim\!1\%$ positive-class binary problem on a temporal student--item interaction stream. A vanilla 2-layer GraphSAGE encoder \citep{hamilton2017inductive} trained on this stream achieves a test PR-AUC of \valMOOCSAGEPR\ (across three seeds), barely above the prevalence floor: the encoder's edge representation does not by itself expose the short-horizon structure around the candidate $(u,v,t)$ that is informative on this task.

\vspace{-0.5em}
\paragraph{One coordinate.}
Now define, informally, a single past-window count, i.e., for a candidate user $u$ at time $t$, the number of distinct items that $u$ has interacted with in the time interval $[t-\Delta,t)$. This is a count of the \emph{star} pattern centered on $u$---one node with several past out-edges to distinct neighbors---and we will later call it $\texttt{star\_u\_out}$ in our $13$-coordinate motif map. Concatenating the four star counts ($\texttt{star\_u\_out}$, $\texttt{star\_u\_in}$, $\texttt{star\_v\_out}$, $\texttt{star\_v\_in}$, the family we will call $A_2$ in Section~\ref{sec:method}) to the GraphSAGE edge representation lifts MOOC PR-AUC from \valMOOCSAGEPR\ to \valFamMOOCATwo, recovering most of the \valMOOCMotifsSAGEPR\ obtained with all $13$ coordinates (Table~\ref{tab:toy_mooc}).

\begin{table}[t]
\centering
\caption{MOOC interaction-classification PR-AUC on the GraphSAGE backbone, mean$\pm$std across three seeds. A single short-horizon motif family ($A_2$, four star counts) recovers most of the full $13$-feature lift. A leave-one-out ablation on the full feature set (Appendix~\ref{app:ablations_loo}, Table~\ref{tab:leave_one_out}) further isolates \texttt{star\_u\_out} as the dominant single coordinate ($\sim\!69\%$ relative drop when removed).}
\label{tab:toy_mooc}
\small\setlength{\tabcolsep}{8pt}
\renewcommand{\arraystretch}{1.1}
\begin{tabular}{l c}
\toprule
\textbf{Setting} & \textbf{Test PR-AUC} \\
\midrule
SAGE (no motifs)                              & \valMOOCSAGEPR \\
SAGE + $A_2$ star counts only ($4$ features)  & \valFamMOOCATwo \\
SAGE + full motif map ($13$ features)         & \textbf{\valMOOCMotifsSAGEPR} \\
\bottomrule
\end{tabular}
\end{table}

\begin{figure}[t]
  \centering
  \includegraphics[width=0.75\linewidth]{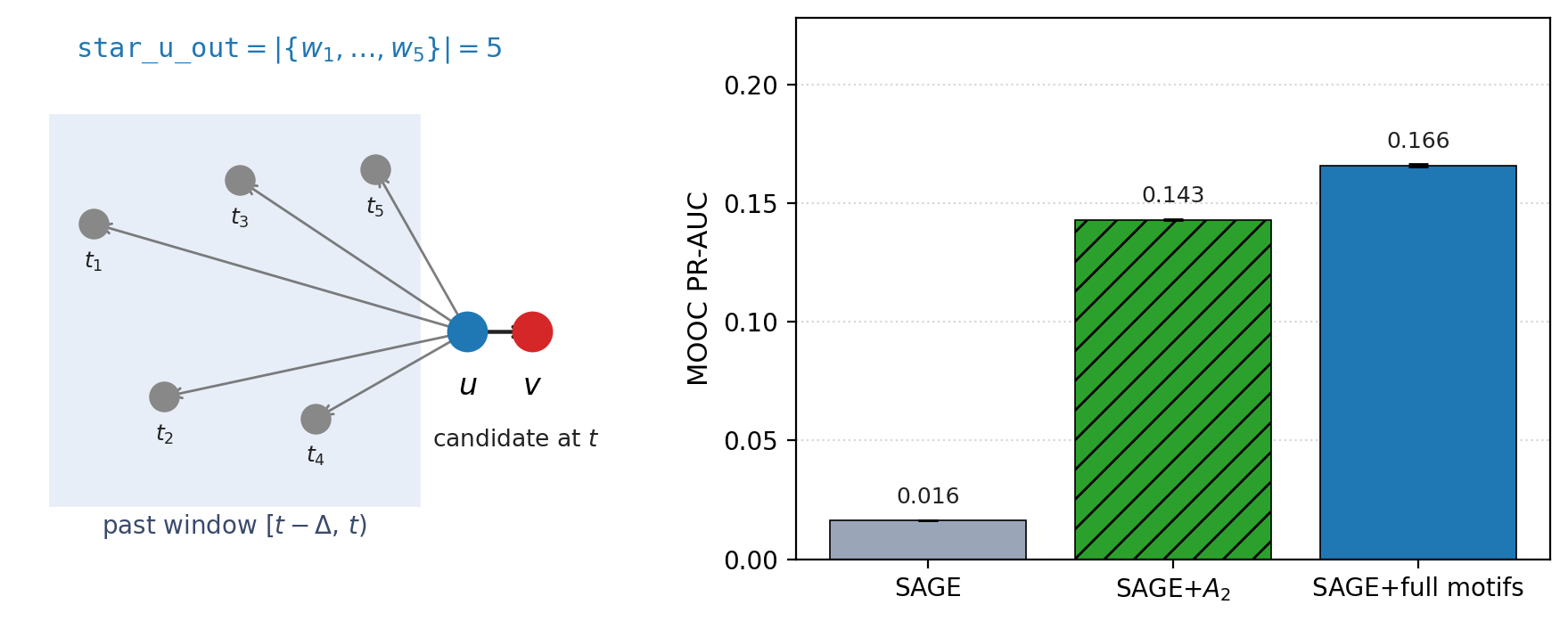}
  \caption{Past-window star pattern around a MOOC candidate, and the resulting PR-AUC ladder. A single short-horizon family of star counts already turns a near-chance task into a useful predictor.}
  \label{fig:toy_mooc}
\end{figure}

\vspace{-0.5em}
\paragraph{Why this matters.}
The interpretation is concrete (Figure~\ref{fig:toy_mooc}): students whose set of recently-touched items is large differ systematically from those whose set is small, and this single piece of past-window structure---expressible as the count of distinct out-neighbors of $u$ in a star pattern---is enough to turn a near-chance task into a useful predictor. A leave-one-out ablation over the full $13$-coordinate feature set (Appendix~\ref{app:ablations_loo}, Table~\ref{tab:leave_one_out}) confirms that $\texttt{star\_u\_out}$ alone accounts for $\sim\!69\%$ of the full motif-augmentation lift on MOOC.

\vspace{-0.5em}
\paragraph{Beyond MOOC.}
Two questions follow: which other short-horizon patterns appear in real and synthetic temporal streams, and how do we package them into a fixed feature family that plugs into existing temporal GNN encoders? Section~\ref{sec:characterization} addresses the first across $13$ datasets, and Section~\ref{sec:method} addresses the second by formalizing the $13$-coordinate motif map.

\vspace{-0.5em}
\section{Empirical diagnostic across temporal streams}
\label{sec:characterization}
\vspace{-0.5em}

The toy example of Section~\ref{sec:toy} showed that a single past-window star count carries strong predictive signal on MOOC. To see whether this is an isolated phenomenon, we compute per-candidate, past-window motif counts---generalizing the single $\texttt{star\_u\_out}$ coordinate of Section~\ref{sec:toy} to a fixed family of $13$ \emph{dyadic}, \emph{star}, and \emph{triadic} coordinates whose precise definitions we defer to Section~\ref{sec:method}---on Bitcoin Alpha, Bitcoin OTC, MOOC, PaySim, the three TGB link-property datasets, and six synthetic generator families (Erd\H{o}s--R\'enyi, Barab\'asi--Albert, Watts--Strogatz, configuration model, stochastic block model, and a degree-preserving rewiring control). Throughout this section we treat these counts as the abstract object $c(u,v,t)\in\mathbb{R}^{13}$ and use them descriptively; their exact definitions, the algorithm, normalization, and encoder interface follow in Section~\ref{sec:method}.

\vspace{-0.5em}
\paragraph{Cross-dataset signatures.}
For each dataset $\mathcal{D}$ and a dataset-specific window radius $\Delta$, we summarize the per-edge count vector $c(u,v,t)$ by its empirical mean $\bar c(\mathcal{D})=\tfrac{1}{N}\sum_i c(u_i,v_i,t_i)$ and covariance $\widehat\Sigma(\mathcal{D})$. Analyzing $\bar c$ and the leading eigendirections of $\widehat\Sigma$ reveals three dominant axes of variation across domains, which we label $A_1$ (dyadic recency/reciprocity), $A_2$ (star diversity, the family containing the MOOC $\texttt{star\_u\_out}$ coordinate of Section~\ref{sec:toy}), and $A_3$ (triadic flow). Figure~\ref{fig:signatures} visualizes the cross-dataset signatures grouped by axis. The axes are not imposed by feature labelling alone: a leading-PCA decomposition of the per-dataset means yields three components whose loadings concentrate, respectively, on the $A_1$, $A_2$, and $A_3$ coordinates, with cumulative variance explained reported in Appendix~\ref{app:extra_signatures}.

\begin{figure}[t]
  \centering
  \includegraphics[width=0.75\linewidth]{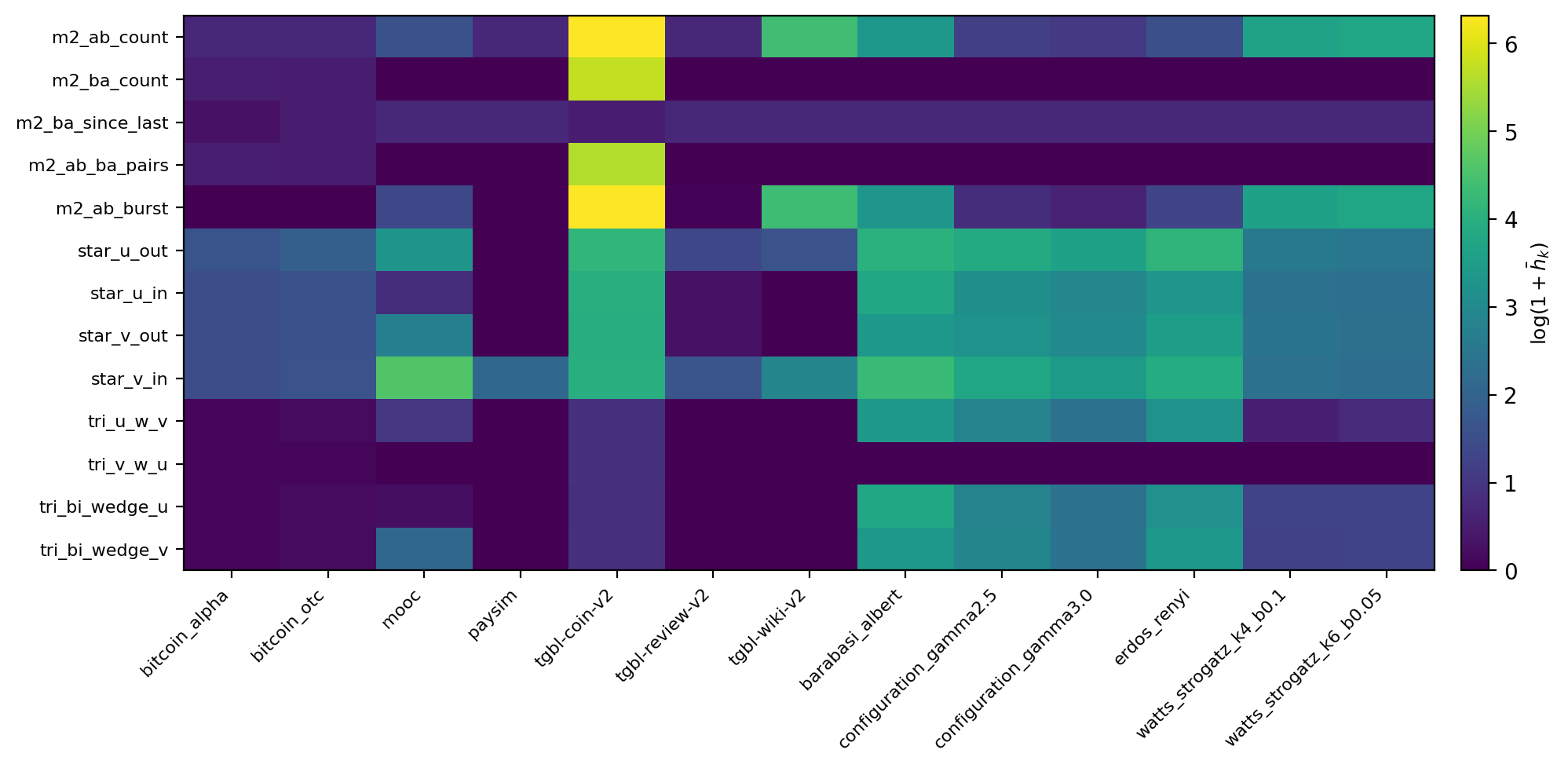}
  \caption{Empirical motif signatures across real and synthetic temporal streams. Rows: $13$ motif coordinates, grouped by family with horizontal separators between $A_1$ (rows 1--5, dyadic recency/reciprocity), $A_2$ (rows 6--9, star diversity), and $A_3$ (rows 10--13, triadic flow); columns: datasets, with vertical separators between real and synthetic streams. The three axes dominate the between-dataset variation in distinct ways and motivate the sub-family decomposition of Section~\ref{sec:method}.}
  \label{fig:signatures}
\end{figure}

We do not claim strict minimality; the leave-one-out sweep (Appendix~\ref{app:ablations_loo}) supports the weaker claim that no single coordinate is uniformly redundant across datasets. Appendix~\ref{app:extra_signatures} additionally verifies that (i) alternative feature families (global motif counts, static-only anchored counts) fail to separate several domain pairs that $c(u,v,t)$ separates; (ii) the three-axis summary is robust to the choice of $\Delta$ across an order of magnitude; and (iii) a linear probe on $\bar c(\mathcal{D})$ recovers the dataset label substantially better than the same probe on global motif counts.

\vspace{-0.5em}
\paragraph{Empirical takeaway.}
Motif activity in real and synthetic temporal streams concentrates on three interpretable axes, and per-dataset signatures along those axes are well-defined properties of the generating process rather than artifacts of a particular draw: most of the $13$ coordinates are substantially more stable across scale perturbations than across domains, with the heavy-tailed dyadic counts the least stable and log-normalized in the formal map of the next section (Appendix~\ref{app:scale_stability}, Figure~\ref{fig:stability}). The per-dataset weights along the three axes will also predict, in Section~\ref{sec:experiments}, where motif augmentation produces the largest gains.

\vspace{-0.5em}
\section{Candidate-local temporal motif feature framework}
\label{sec:method}
\vspace{-0.5em}

We now operationalize Section~\ref{sec:characterization}'s three axes as a small, fixed feature family: per-candidate and past-only (leakage safety); fixed dimension (concatenation with any encoder); and three sub-families in one-to-one correspondence with the empirically dominant axes $(A_1,A_2,A_3)$.

\vspace{-0.5em}
\paragraph{Past-only temporal windows.}
Fix a window radius $\Delta>0$. For a reference time $t$ the \emph{past-only window} is
\begin{equation}
  \mathcal{W}^{\mathrm{past}}_t(\Delta) \;=\; \bigl\{\,i\in\{1,\dots,N\} : t-\Delta \le t_i < t\,\bigr\}.
\end{equation}
Per-node time-sorted in/out event lists allow binary search to retrieve $\mathcal{W}^{\mathrm{past}}_t(\Delta)$ restricted to any node in $O(\log N)$ time \citep{paranjape2017motifs}. The strict inequality $t_i<t$ rules out events at the candidate time, and is the first of the two leakage guards used throughout the paper.

\vspace{-0.5em}
\paragraph{Edge-anchored motif feature map.}
For a candidate $(u,v,t)$ we define a vector $h(u,v,t)\in\mathbb{R}^{13}$ assembled from three sub-families that realize the three axes of Section~\ref{sec:characterization}:

\textit{$A_1$ Dyadic recency/reciprocity (5 features):} counts of past $u\!\to\! v$ and $v\!\to\! u$ events in $\mathcal{W}^{\mathrm{past}}_t(\Delta)$, the minimum of these counts (paired events), the excess of repeat $u\!\to\! v$ events beyond the first (burstiness), and the normalized gap since the last reciprocal $v\!\to\! u$ event.

\textit{$A_2$ Star diversity (4 features):} the sizes of the four sets $\mathcal{N}^{\mathrm{out}}_u,\mathcal{N}^{\mathrm{in}}_u,\mathcal{N}^{\mathrm{out}}_v,\mathcal{N}^{\mathrm{in}}_v$ of distinct neighbors of $u$ and $v$ in the window, excluding $\{u,v\}$ to avoid degeneracy and optionally capped at $C$ to bound work in hubs. The MOOC \texttt{star\_u\_out} coordinate of Section~\ref{sec:toy} is the first of these four.

\textit{$A_3$ Triadic flow (4 features):} the sizes of the two directed cross-wedge intersections $|\mathcal{N}^{\mathrm{out}}_u\cap\mathcal{N}^{\mathrm{in}}_v|$, $|\mathcal{N}^{\mathrm{out}}_v\cap\mathcal{N}^{\mathrm{in}}_u|$, and the two bi-wedge intersections $|\mathcal{N}^{\mathrm{in}}_u\cap\mathcal{N}^{\mathrm{in}}_v|$, $|\mathcal{N}^{\mathrm{out}}_u\cap\mathcal{N}^{\mathrm{out}}_v|$.

\vspace{-0.5em}
\paragraph{Normalization, neighbor cap, and complexity.}
All count coordinates except \texttt{m2\_ba\_since\_last} pass through the per-coordinate transform $x\!\mapsto\!\log_{10}(1+x)$ before entering the linear embedding $M$ below; this is part of the method, not a postprocessing convenience, and is what makes the heavy-tailed dyadic counts comparable across coordinates. To bound work at hubs, neighbor sets exceeding a cap $C$ (default $C{=}100$) are subsampled \emph{deterministically}: the random subset is drawn from a numpy generator seeded by a hash of $(u,v,t)$, so $h(u,v,t)$ is a deterministic, past-only function of the stream and the candidate alone, computed in $O(\log N + C)$ time per candidate. Together with the second leakage guard (static encoders propagate messages over training edges only; temporal encoders discard memory updates past $t$ at evaluation time), this guarantees that motif features for negative candidates use only the same past window as the corresponding positive. Appendix~\ref{app:motif_defs} gives full formal definitions, pseudocode, and the per-dataset cap-bind frequency.

\vspace{-0.5em}
\paragraph{Window-radius selection.}
The single tunable hyperparameter, $\Delta$, is selected from a small grid of $\{10^k : k=1,\dots,8\}$ seconds (rounded to the natural unit of each dataset) on the validation split. Per-dataset selected $\Delta$s, together with a robustness check that selecting $\Delta$ from quantiles of inter-event times yields the same coarse choice, are reported in Appendix~\ref{app:protocol}.

\vspace{-0.5em}
\paragraph{Linear motif embedding and encoder interface.}
A learnable matrix $M\in\mathbb{R}^{13\times d_m}$ maps the count vector to $z_{\mathrm{motif}}(u,v,t)=h(u,v,t)^\top M$. For a static base encoder $g_\theta$ (e.g., GraphSAGE, GCN, GAT) the concatenated edge representation $[z_u,z_v,|z_u-z_v|,\phi_\eta(x(u,v,t)),z_{\mathrm{motif}}(u,v,t)]$ is fed to an MLP head; for a temporal base encoder (TGN, TPNet, DyGFormer, TNCN, HyperEvent), $z_{\mathrm{motif}}(u,v,t)$ is concatenated to the encoder's time-dependent edge representation before the prediction head, leaving the encoder's internal memory or attention unchanged.

\vspace{-0.5em}
\section{A temporal Weisfeiler-Leman placement}
\label{sec:theory}
\vspace{-0.5em}

This section places the motif feature map of Section~\ref{sec:method} relative to a candidate-anchored temporal expressivity hierarchy. We work with \emph{anchored pointed streams} $(\mathcal{D},(u_0,v_0),t)$ identifying a candidate edge at reference time $t$. Define \emph{temporal-$1$-WL} as the standard $1$-WL color refinement extended with a per-message time-bucket label $\mathsf{buck}_\Delta(t-t_i)\in\{1,\dots,B\}$ at fixed quantization $B$, and write $C^{(\ell)}=(c^{(\ell)}_{u_0},c^{(\ell)}_{v_0})$ for the resulting \emph{anchor-pair color} at depth $\ell$ (formal definition: Appendix~\ref{app:theory}, Definition~\ref{def:t1wl}). Under this abstraction, motif features distinguish anchored streams that any temporal-$1$-WL-equivalent edge scorer collapses. The role of the result is to locate the augmentation, not to explain the empirical gains of Section~\ref{sec:experiments}.

\begin{lemma}[Temporal MP-GNN edge scorers are bounded by temporal-$1$-WL]
\label{lem:mpgnn_bounded}
Let $f$ be any temporal message-passing edge scorer that, for a candidate $(u_0,v_0,t)$, computes per-node states via per-layer updates of the form $h_v^{(\ell)}(t)=\psi^{(\ell)}\!\bigl(h_v^{(\ell-1)}(t),\mathrm{AGG}\,\{\phi^{(\ell)}(h_v^{(\ell-1)}(t),h_{u'}^{(\ell-1)}(t),x_{vu'},\tau(t-t_i))\}\bigr)$ over events in $\mathcal{W}^{\mathrm{past}}_t(\Delta)$ via a permutation-invariant aggregator and outputs a score $s(u_0,v_0,t)=\rho\bigl(h_{u_0}^{(T)}(t),h_{v_0}^{(T)}(t)\bigr)$, for any measurable time-encoding $\tau$ with finite range. If temporal-$1$-WL with $B$ buckets matching the discretization of $\tau$ assigns the same anchor-pair color at depth $T$ to two anchored pointed streams, then $f$ assigns them the same score.
\end{lemma}

The proof (Appendix~\ref{app:proof_mpgnn_bounded}) adapts \citet{xu2018powerful,morris2019weisfeiler} to the temporal multiset setting. The quantization assumption matches how TGN, TGAT, DyGFormer and GraphMixer bin or embed times in practice \citep{rossi2020temporal,xu2020inductive,yu2023towards,cong2023graphmixer} and is sharpened in the appendix to a bounded-Lipschitz version with an $O(L\varepsilon)$ approximation gap. A closely related placement appears in \citet{souza2022provably}. The lemma covers the standard temporal MP-GNN edge-scoring abstraction underlying our baselines but not, e.g., models with pair-dependent subgraph extraction or explicit motif enumeration.

\begin{theorem}[Candidate-anchored strict separation]
\label{thm:separation}
There exist two anchored pointed temporal streams $(\mathcal{D}_1,(u_0,v_0),t_1)$ and $(\mathcal{D}_2,(u_0',v_0'),t_2)$ such that (1) no isomorphism between $\mathcal{D}_1$ and $\mathcal{D}_2$ maps the anchor of one to the anchor of the other; (2) temporal-$1$-WL produces equal anchor-pair colors at every refinement depth $\ell$ and every bucketization $B$; (3) $h(u_0,v_0,t_1)\ne h(u_0',v_0',t_2)$, so any edge scorer that injectively reads $h$ distinguishes the two anchored streams.
\end{theorem}

The witness is the classical $C_6$ vs.\ $2K_3$ pair lifted to anchored directed temporal streams: temporal-$1$-WL collapses the two for every $\ell$ and every $B$, while the four $A_3$ triadic-flow coordinates of $h$ take value $0$ on $C_6$ and $1$ on $2K_3$, isolating $A_3$ as the separator. The full witness construction and proof, a strictly-temporal variant with $K$ distinct timestamps, and a temporal-$k$-WL hierarchy that grades motif families by their WL category are in Appendix~\ref{app:theory}.

\vspace{-0.5em}
\section{Experiments}
\label{sec:experiments}
\vspace{-0.5em}

Section~\ref{sec:tgb} measures Contribution~3 on TGB link-property prediction; Section~\ref{ssec:edge_class} reports edge classification on Bitcoin Alpha/OTC and MOOC; Section~\ref{ssec:graph_class} reports a synthetic generator graph-level task; Sections~\ref{ssec:ablations}--\ref{ssec:controls} report ablations and controls. PaySim, where seed-level variance exceeds the augmentation effect, is reported as a stress test in Appendix~\ref{app:paysim}. All experiments respect two leakage guards: (a) motif counts come only from $\mathcal{W}^{\mathrm{past}}_t(\Delta)$ with $t_i<t$; and (b) static encoders use training edges only for message passing, while temporal encoders discard memory updates past $t$ at evaluation time. Splits are $80/10/10$ chronological for non-TGB data and the official splits for TGB. Baseline numbers use the official released hyperparameters; the motif-augmented variant inherits them verbatim and additionally selects $\Delta$ on the validation split. Per-dataset hyperparameters, architectures, optimizer, negative sampling protocols, and random seeds are in Appendix~\ref{app:protocol}.

\vspace{-0.5em}
\subsection{TGB link property prediction}
\label{sec:tgb}
\vspace{-0.5em}

Table~\ref{tab:tgb-linkprop} reports test MRR on the three TGB link-property benchmarks (\texttt{tgbl-wiki-v2}, \texttt{tgbl-review-v2}, \texttt{tgbl-coin-v2}) for four strong TGNN baselines \citep{rossi2020temporal,yu2023towards,yu2024temporalwalk,gao2025hyperevent} and their motif-augmented counterparts. The task in every block is the official TGB link-property prediction task: rank a positive candidate edge $(u,v,t)$ against the official negative samples by mean reciprocal rank. Rows are ordered by reproduced baseline rank. Headline rows (DyGFormer / TNCN on review, DyGFormer on coin) use five seeds; the remaining rows use three, and we report $\pm 1$ standard deviation throughout with no implicit hypothesis-test claims. Baseline numbers are our own reproductions on the official splits using the released baseline code; Appendix~\ref{app:tgb_leaderboard} positions them against the TGB leaderboard, and every motif-augmentation gain is measured against \emph{our} reproduced baseline rather than an external best.

\begin{table}[t]
\centering
\caption{TGB link-property prediction (test MRR, official splits and evaluator). Higher is better; the task is to rank a positive $(u,v,t)$ above the official negatives. Augmentation lifts every $(\text{baseline}, \text{dataset})$ pair we evaluate, with the largest gain DyGFormer's $0.224\!\to\!0.413$ on \texttt{tgbl-review-v2} (and a sizeable secondary $0.752\!\to\!0.819$ on \texttt{tgbl-coin-v2}); two pairs (TPNet on coin, HyperEvent on coin) are within one standard deviation of zero, consistent with those baselines already capturing the relevant short-horizon structure. Validation MRRs are in Appendix~\ref{app:tgb_leaderboard}.}
\label{tab:tgb-linkprop}
\small
\setlength{\tabcolsep}{4pt}
\renewcommand{\arraystretch}{1.05}
\begin{tabular}{c l c c c}
\toprule
\textbf{Rank} & \textbf{Method} & \texttt{tgbl-wiki-v2} & \texttt{tgbl-review-v2} & \texttt{tgbl-coin-v2} \\
\midrule
1 & TPNet      & \mrr{0.827}{0.001} & \na                & \mrr{0.832}{0.001} \\
\rowcolor{gray!12}
  & \textbf{+ Motifs (TPNet)}      & \mrr{0.829}{0.001} & \na                & \mrr{0.832}{0.001} \\
2 & HyperEvent & \mrr{0.810}{0.002} & \mrr{0.268}{0.004} & \mrr{0.773}{0.002} \\
\rowcolor{gray!12}
  & \textbf{+ Motifs (HyperEvent)} & \mrr{0.819}{0.001} & \mrr{0.275}{0.003} & \mrr{0.772}{0.001} \\
3 & DyGFormer  & \mrr{0.798}{0.004} & \mrr{0.224}{0.015} & \mrr{0.752}{0.004} \\
\rowcolor{gray!12}
  & \textbf{+ Motifs (DyGFormer)}  & \mrr{0.828}{0.003} & \textbf{\mrr{0.413}{0.002}} & \mrr{0.819}{0.002} \\
4 & TNCN       & \mrr{0.718}{0.001} & \mrr{0.377}{0.010} & \mrr{0.762}{0.004} \\
\rowcolor{gray!12}
  & \textbf{+ Motifs (TNCN)}       & \mrr{0.736}{0.002} & \mrr{0.424}{0.003} & \mrr{0.772}{0.003} \\
5 & TGN        & \na                & \mrr{0.349}{0.020} & \na                \\
\rowcolor{gray!12}
  & \textbf{+ Motifs (TGN)}        & \na                & \mrr{0.418}{0.016} & \na                \\
\bottomrule
\end{tabular}
\end{table}

\vspace{-0.5em}
\paragraph{Where does the gain come from?}
The pattern of lifts varies systematically: DyGFormer on review shows the largest gain ($0.224\!\to\!0.413$); TPNet/HyperEvent on coin are within one standard deviation of zero; wiki gains are positive but modest ($+0.01$ to $+0.04$ MRR). Two pieces of evidence connect this pattern to the motif structure of Section~\ref{sec:method} rather than to extra capacity or weak baselines: (i) a candidate-type analysis (Appendix~\ref{app:candidate_split}) shows the lift concentrates on repeated-pair and triadic-witness candidates---exactly the populations with non-zero motif coordinates; and (ii) the controls of Section~\ref{ssec:controls} rule out extra-capacity and GNN-redundancy explanations, with the per-dataset choice of dominant sub-control tracking the $A_1$/$A_2$/$A_3$ signature of Section~\ref{sec:characterization}.

\vspace{-0.5em}
\subsection{Edge classification}
\label{ssec:edge_class}
\vspace{-0.5em}

Tables~\ref{tab:bitcoin-edge-class}--\ref{tab:mooc-edge-class} show the effect of motif augmentation on Bitcoin Alpha/OTC trust classification and MOOC interaction classification, two settings where short-horizon structure is the dominant predictive signal. We carry over the static encoders from \citep{kipf2017semi,hamilton2017inductive,velickovic2018graph} and add TGN \citep{rossi2020temporal} as a strong temporal baseline. Bitcoin Alpha/OTC are framed as a sanity check (the underlying networks are small, and the task is largely whole-history reciprocity prediction); MOOC, with positive-class prevalence $\sim 1\%$, is the more demanding test. PaySim is treated as a noisy stress test and is reported in Appendix~\ref{app:paysim} together with extra-seed runs and a discussion of why its standard deviations exceed the augmentation effect for SAGE/GCN/GAT.

\begin{table}[t]
\centering
\caption{Bitcoin Alpha/OTC edge classification: predict whether a directed candidate edge $(u,v,t)$ in the trust-rating stream carries a positive (trusted) rather than negative (distrusted) label, given only past events. PR-AUC is the headline metric (positive class is the minority); ROC-AUC reported for context. Mean$\pm$std across three seeds.}
\label{tab:bitcoin-edge-class}
\small
\setlength{\tabcolsep}{5pt}
\renewcommand{\arraystretch}{1.1}
\begin{tabular}{c l c c}
\toprule
 & \textbf{Model} & \textbf{PR-AUC} & \textbf{ROC-AUC} \\
\midrule
\multirow{6}{*}{\rotatebox{90}{Alpha}}
 & SAGE             & \valAlphaSAGEPR & \valAlphaSAGEROC \\
 & GCN              & \valAlphaGCNPR  & \valAlphaGCNROC  \\
 & GAT              & \valAlphaGATPR  & \valAlphaGATROC  \\
 & TGN              & \valAlphaTGNPR  & \valAlphaTGNROC  \\
\rowcolor{gray!12}
 & \textbf{+ Motifs (SAGE)} & \textbf{\valAlphaMotifsSAGEPR} & \textbf{\valAlphaMotifsSAGEROC} \\
\rowcolor{gray!12}
 & \textbf{+ Motifs (TGN)}  & \textbf{\valAlphaMotifsTGNPR}  & \textbf{\valAlphaMotifsTGNROC}  \\
\cmidrule(lr){2-4}
\multirow{6}{*}{\rotatebox{90}{OTC}}
 & SAGE             & \valOTCSAGEPR & \valOTCSAGEROC \\
 & GCN              & \valOTCGCNPR  & \valOTCGCNROC  \\
 & GAT              & \valOTCGATPR  & \valOTCGATROC  \\
 & TGN              & \valOTCTGNPR  & \valOTCTGNROC  \\
\rowcolor{gray!12}
 & \textbf{+ Motifs (SAGE)} & \textbf{\valOTCMotifsSAGEPR} & \textbf{\valOTCMotifsSAGEROC} \\
\rowcolor{gray!12}
 & \textbf{+ Motifs (TGN)}  & \textbf{\valOTCMotifsTGNPR}  & \textbf{\valOTCMotifsTGNROC}  \\
\bottomrule
\end{tabular}
\end{table}

\begin{table}[t]
\centering
\caption{MOOC edge classification: predict, on a continuous-time student--item interaction stream, whether each interaction is followed by the rare ($\sim\!1\%$) ``drop-out'' event class. PR-AUC is the headline metric (PR-AUC is the area under the precision--recall curve and is strict on rare-positive tasks); ROC-AUC reported for context. Motif augmentation produces a large absolute and relative PR-AUC lift on every backbone in this table. Mean$\pm$std across three seeds.}
\label{tab:mooc-edge-class}
\small
\setlength{\tabcolsep}{5pt}
\renewcommand{\arraystretch}{1.1}
\begin{tabular}{lcc}
\toprule
\textbf{Model} & \textbf{PR-AUC} & \textbf{ROC-AUC} \\
\midrule
SAGE             & \valMOOCSAGEPR & \valMOOCSAGEROC \\
GCN              & \valMOOCGCNPR  & \valMOOCGCNROC  \\
GAT              & \valMOOCGATPR  & \valMOOCGATROC  \\
TGN              & \valMOOCTGNPR  & \valMOOCTGNROC  \\
\rowcolor{gray!12}
\textbf{+ Motifs (SAGE)} & \textbf{\valMOOCMotifsSAGEPR} & \textbf{\valMOOCMotifsSAGEROC} \\
\rowcolor{gray!12}
\textbf{+ Motifs (TGN)}  & \textbf{\valMOOCMotifsTGNPR}  & \textbf{\valMOOCMotifsTGNROC}  \\
\bottomrule
\end{tabular}
\end{table}

The size of these lifts---modest on Bitcoin Alpha/OTC, large on MOOC---tracks the per-dataset signatures of Section~\ref{sec:characterization}: trust streams already encode whole-history reciprocity in their static structure, while MOOC's short-horizon star pattern is precisely what the static GNN baseline is missing.

\vspace{-0.5em}
\subsection{Graph-level classification on synthetic temporal generators}
\label{ssec:graph_class}
\vspace{-0.5em}

As corroborating evidence that the same $13$-coordinate motif map carries useful signal at the graph level, we evaluate motif augmentation on a synthetic graph-classification task with six temporal generator families (Erd\H{o}s--R\'enyi \citep{erdos1960evolution}, Barab\'asi--Albert \citep{barabasi1999emergence}, Watts--Strogatz \citep{watts1998collective}, configuration model \citep{newman2001random}, stochastic block model, and a rewiring control). Augmentation concatenates the mean of $z_{\mathrm{motif}}$ across graph edges to the graph-level readout. Across five seeds it lifts SAGE from $\sim\!59\%$ to $\sim\!92\%$ accuracy, GCN from $\sim\!73\%$ to $\sim\!93\%$, and most strikingly GAT from $16.7\%$ (majority-class collapse) to $94\%$---without architectural changes; raw-motif RF/SVM baselines and the full table appear in Appendix~\ref{app:graph_class_main}. Distinguishing six generator families is a relatively easy task whenever any short-horizon structural axis is informative, so we treat this as supporting evidence rather than a separate empirical claim.

\vspace{-0.5em}
\subsection{Ablations}
\label{ssec:ablations}
\vspace{-0.5em}

A per-motif leave-one-out sweep (Appendix~\ref{app:ablations_loo}) zeroes out each of the $13$ coordinates in turn; the largest single-coordinate impact is the MOOC $\texttt{star\_u\_out}$ effect of Section~\ref{sec:toy} ($\sim$$-69\%$ relative PR-AUC), with smaller dataset-specific drops elsewhere. A per-family ablation on the SAGE backbone (Appendix~\ref{app:family_ablation}) sweeps $(A_1,A_2,A_3)$ and their pairwise combinations: on Alpha, $A_1$ alone reaches \valFamAlphaAOne, essentially equal to All (\valFamAlphaAAll), consistent with the trust signature being $A_1$-dominated; on MOOC, $A_2$ alone (\valFamMOOCATwo) recovers most of All (\valFamMOOCAAll), with the residual closed by adding $A_1$. The window radius $\Delta$ has a clear sweet spot per dataset (too small undercounts triadic features; too large saturates around hubs), and replacing hard counts by exponentially-decayed counts $\kappa(\delta)=e^{-\delta/\tau}$ matches or marginally improves them; details, plots, and the PaySim row are in Appendix~\ref{app:family_ablation}, \ref{app:soft_vs_hard}, and \ref{app:paysim}.

\vspace{-0.5em}
\subsection{Controls}
\label{ssec:controls}
\vspace{-0.5em}

We rule out two classes of skeptical explanations of the motif-induced gain in Section~\ref{sec:tgb}---extra capacity and GNN redundancy---and use two further sub-controls (activity-only, recency-only) to localize, dataset by dataset, which motif sub-family the gain is concentrated in.

\vspace{-0.5em}
\paragraph{Random-feature and motif-only controls.}
We run two controls on the SAGE backbone for the edge-classification datasets (Table~\ref{tab:motif_only_random}). The \emph{random-feature} control replaces $h(u,v,t)$ by a deterministic per-candidate $13$-dimensional Gaussian whose per-coordinate mean and standard deviation are matched to the real motif histogram, with the rest of the pipeline (motif embedding $M$, head MLP, encoder) held fixed; matching motif augmentation would mean the lift is from extra capacity. The \emph{motif-only} head MLP strips the GNN and feeds $h(u,v,t)$ directly into a $13\!\to\!64\!\to\!1$ head; matching motif augmentation would mean the GNN is essentially redundant. The random-feature control matches the no-motif baseline within one standard deviation on Alpha and MOOC (\valRandomFeatAlpha\ vs \valAlphaSAGEPR; \valRandomFeatMOOC\ vs \valMOOCSAGEPR), and on PaySim falls within the seed-to-seed instability band that the baseline already shows on this stress-test dataset; the motif-only head recovers a substantial fraction of the motif augmentation lift on Alpha (\valMotifOnlyAlpha) and MOOC (\valMotifOnlyMOOC) but collapses on PaySim (\valMotifOnlyPaySim), consistent with the dataset-level $A_3$ near-vanishing on PaySim (Appendix~\ref{app:paysim}). A separate motif-baseline comparison (global counts vs.\ static-only anchoring vs.\ our temporal anchoring; Appendix~\ref{app:motif_baselines}) shows static anchoring competitive on Alpha and temporal anchoring strictly best on MOOC, matching the per-dataset signatures of Section~\ref{sec:characterization}.

\vspace{-0.5em}
\paragraph{Activity-only and recency-only controls.}
A further skeptical hypothesis is that motif augmentation is a complicated way to expose simple node-activity counts (degree in the past window) or pair recency to the head. We replace $h(u,v,t)$ by two strict subsets and re-run with the same head: an \emph{activity-only} control that exposes only the four $A_2$ scalar counts, and a \emph{recency-only} control that exposes only \texttt{m2\_ba\_since\_last}. The two edge-classification datasets give complementary readings, both consistent with the per-dataset signature of Section~\ref{sec:characterization}: on Alpha (whole-history reciprocity) recency-only alone matches motif augmentation within one standard deviation while activity-only is below; on MOOC (short-horizon star diversity) activity-only captures the bulk of the lift but motif augmentation adds a measurable margin on top, while recency-only is at the no-motif baseline (Appendix~\ref{app:activity_recency_controls}). The lift is therefore not a single rebundled scalar---different sub-controls dominate on different datasets.

\vspace{-0.5em}
\paragraph{Overhead and interpretability.}
Motif extraction adds a one-shot preprocessing cost that is small relative to a single training epoch on TGB-scale streams; per-epoch training overhead and memory are well within budget (Appendix~\ref{app:overhead}). The learned linear embedding $M$ recovers the three-axis structure of Section~\ref{sec:characterization} from the trained model: trust networks concentrate mass on $A_1$ and MOOC on $A_2,A_3$, mirroring the data-side signatures (Appendix~\ref{app:interpretability}).

\begin{table}[t]
\centering
\caption{Random-feature and motif-only controls on the SAGE backbone (PR-AUC, mean$\pm$std across three seeds). ``Random'' replaces $h$ by moment-matched Gaussian features in the full pipeline; ``motif-only'' is a $13\!\to\!64\!\to\!1$ MLP on $h(u,v,t)$ with no GNN. Tasks are the edge-classification tasks of Section~\ref{ssec:edge_class} (Bitcoin Alpha trust prediction, MOOC drop-out prediction) and the PaySim fraud-detection stress test (Appendix~\ref{app:paysim}).}
\label{tab:motif_only_random}
\small
\begin{tabular}{l ccc}
\toprule
Task & Baseline (no motifs) & + Random features & + Motif-only MLP \\
\midrule
Alpha (PR-AUC)            & \valAlphaSAGEPR  & \valRandomFeatAlpha  & \valMotifOnlyAlpha \\
MOOC (PR-AUC)             & \valMOOCSAGEPR   & \valRandomFeatMOOC   & \valMotifOnlyMOOC \\
PaySim (PR-AUC)           & \valPaySimSAGEPR & \valRandomFeatPaySim & \valMotifOnlyPaySim \\
\bottomrule
\end{tabular}
\end{table}

\vspace{-0.5em}
\section{Related work}
\label{sec:related}
\vspace{-0.5em}

Prior work on temporal motifs and GNNs leaves three gaps that this paper directly targets.

\vspace{-0.5em}
\paragraph{Gap 1: Diagnostics of temporal motifs in real data.} Network motifs \citep{milo2002network,alon2007network,benson2016higher} and their $\delta$-temporal extension \citep{paranjape2017motifs} have been used as descriptive statistics or fed into downstream prediction \citep{qiu2021temporal,liu2025motifs} and post-hoc TGNN explanation \citep{chen2023tempme}, but typically on a handful of datasets and without a cross-domain account of which patterns are concentrated where. Section~\ref{sec:characterization} closes this gap with a diagnostic isolating three scale-stable axes ($A_1,A_2,A_3$) that can predict where augmentation helps.

\vspace{-0.5em}
\paragraph{Gap 2: Theory for temporal motif-based GNNs.} Motif-aware GNN architectures \citep{monti2018motifnet,peng2018graph,lee2019graph,chen2023motif,sheng2024mgats} modify the encoder but rarely come with WL-style placements, and the standard subgraph-count augmentation theory \citep{xu2018powerful,morris2019weisfeiler,bouritsas2022improving,frasca2022understanding} is static. \citet{souza2022provably} bound temporal MP-GNNs but do not isolate features that could lift them off that bound. Section~\ref{sec:theory} sharpens this with a candidate-anchored temporal-WL hierarchy, an explicit witness pair on which motif features separate, and a motif-order/WL-hierarchy correspondence (Appendix~\ref{app:theory}).

\vspace{-0.5em}
\paragraph{Gap 3: When do temporal motifs help?} Existing motif-augmented temporal models \citep{qiu2021temporal,liu2025motifs} report aggregate gains without explaining per-dataset variation, and they typically modify the encoder. We instead treat the motif-induced gain as the headline quantity, measured per baseline$\times$dataset pair across five TGNN baselines \citep{rossi2020temporal,yu2023towards,yu2024temporalwalk,gao2025hyperevent} and three task families (edge classification, TGB link prediction, graph-level classification), and connect its size to the $A_1$/$A_2$/$A_3$ signature of Section~\ref{sec:characterization}. The augmentation itself is encoder-agnostic---a per-candidate feature concatenated to any base encoder \citep{trivedi2018representation,kumar2019predicting,xu2020inductive,wang2021inductive,cong2023graphmixer}, and never an architectural change. Appendix~\ref{app:extended_related} gives an extended discussion.

\vspace{-0.5em}
\section{Limitations, broader impact, and conclusion}
\label{sec:limitations}
\vspace{-0.5em}

\textit{Limitations.} (i) The window radius $\Delta$ is a sensitive hyperparameter (Appendix~\ref{app:family_ablation}); multi-scale variants (Appendix~\ref{app:protocol}) trade parameters for robustness and are needed for dependencies beyond $[t-\Delta,t)$. (ii) In sparse streams where windows rarely contain more than one incident event, the triadic-flow $A_3$ coordinates of $h$ identically vanish and the augmentation reduces to its $A_1$+$A_2$ subspace; this is a regime where one would not expect motif augmentation to help, and is consistent with the small or zero gains observed on \texttt{tgbl-coin-v2}/TPNet and (in the appendix) on PaySim. (iii) Lemma~\ref{lem:mpgnn_bounded} assumes a finite-bucket or bounded-Lipschitz time kernel (Appendix~\ref{app:proof_mpgnn_bounded}). (iv) Performance on extreme-frequency streams (e.g., packet-level traces) is open. (v) The motif map $h$ is fixed at order $\le 3$ vertices and a single window $\Delta$; a natural extension increases the motif order or stacks $h$ at multiple $\Delta$s, climbs the temporal-WL hierarchy of Appendix~\ref{app:theory}, and is most likely to help on streams whose dominant predictive structure is multi-scale or higher-order.

\textit{Broader impact.} Temporal motif features suit fraud and anomaly detection, trust and reputation, and transaction monitoring, but by the same token enable deanonymization and interaction-microstructure surveillance; deployments in sensitive domains should pair them with aggregation thresholds, timestamp-level differential privacy, and explicit consent frameworks.

\textit{Conclusion.} Starting from a worked MOOC example and an empirical characterization of motif activity across $13$ temporal datasets, we operationalized the resulting three axes as a $13$-dimensional motif feature family, placed it at the first level of a candidate-anchored temporal Weisfeiler-Leman hierarchy, and measured the motif-induced gain on top of strong TGNN baselines across edge classification, TGB link prediction, and graph-level tasks. The gain is dataset-dependent: large where short-horizon motif structure is part of the predictive signal (e.g., DyGFormer's $0.224\!\to\!0.413$ MRR lift on \texttt{tgbl-review-v2}), modest on \texttt{tgbl-wiki-v2}, and within one standard deviation of zero where the baseline already captures the relevant short-horizon coordinates. The resulting framework is architecture-agnostic, interpretable, and consistently improves downstream performance.


\begin{ack}
We thank the Oxford-Man Institute of Quantitative Finance for the financial support of the work behind this paper.
\end{ack}

\bibliographystyle{plainnat}
\bibliography{references}

\newpage
\appendix
%
\section{Diagnostic supplements}
\label{app:diagnostic_supp}

This appendix complements Section~\ref{sec:characterization} of the main text with two extra views of the per-candidate motif signatures: a UMAP projection of per-edge motif vectors, and a quantitative scale-stability check.

\subsection{Extra signature analyses}
\label{app:extra_signatures}

Figure~\ref{fig:umap} projects per-edge motif vectors $h(u,v,t)$ from each dataset onto two dimensions; we use UMAP where available and fall back to PCA on machines without \texttt{umap-learn}. Real and synthetic streams form clearly separated clusters, complementing the per-domain mean signatures of Fig.~\ref{fig:signatures} and the static- vs.\ temporal-anchored comparison of Table~\ref{tab:motif_baselines}.

\begin{figure}[ht]
  \centering
  \includegraphics[width=0.7\linewidth]{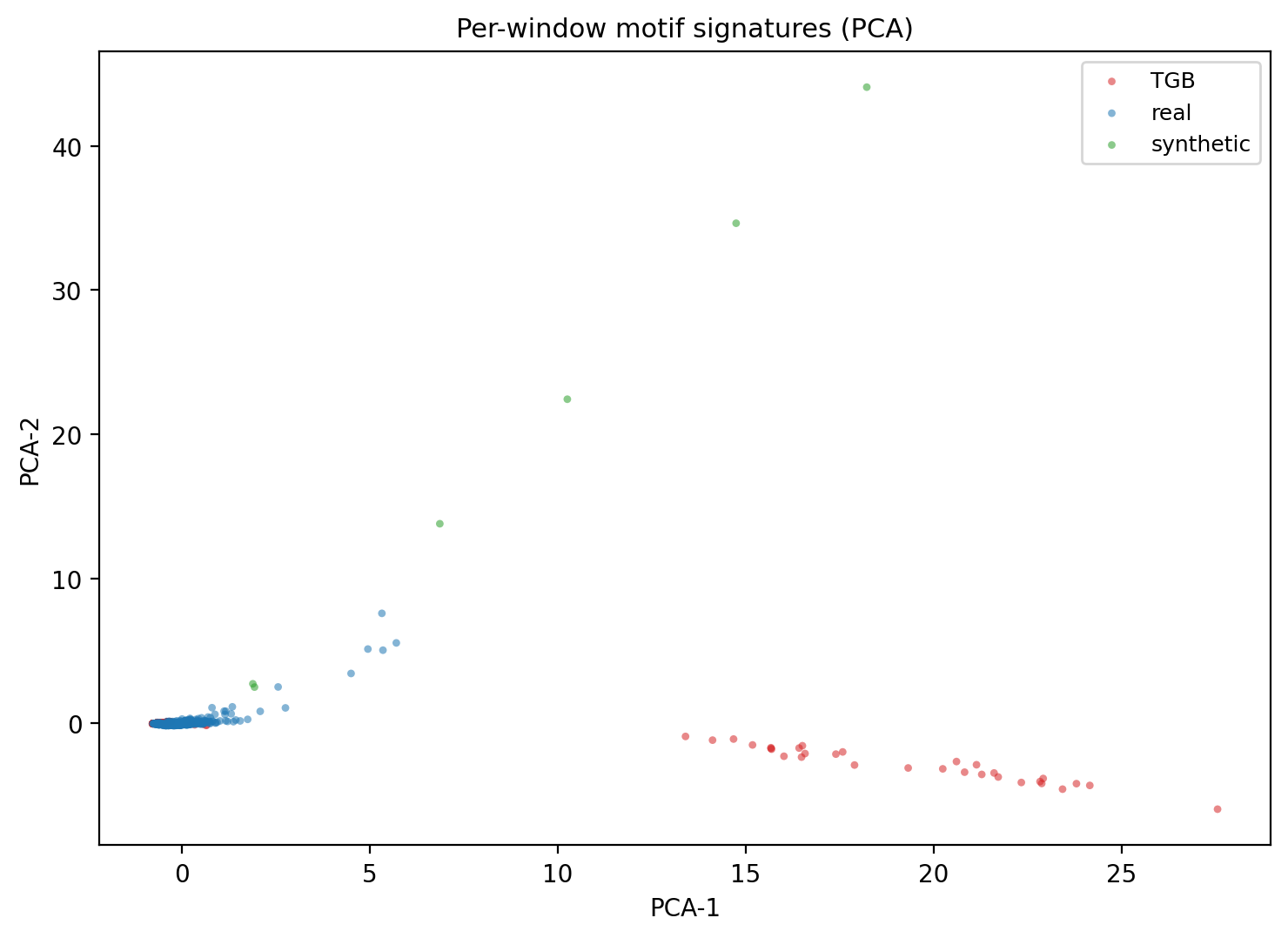}
  \caption{Two-dimensional projection of per-edge motif vectors $h(u,v,t)$, colored by dataset. Temporal-anchored signatures cluster tightly within a domain and separate cleanly across domains.}
  \label{fig:umap}
\end{figure}

\subsection{Stability of motif signatures across network scale}
\label{app:scale_stability}

A useful feature family must be stable as the same generating process produces streams of different sizes. Figure~\ref{fig:stability} compares signatures on the small ($10^3$ nodes, $\sim\!10^4$ events) and large ($10^4$ nodes, $\sim\!10^5$ events) variants of each synthetic generator and on temporally-subsampled versions of the real datasets, reporting the per-coordinate discrepancy normalized by per-coordinate variance. Most coordinates are substantially more stable under scale perturbations than across domains; the dyadic count coordinates are the least stable and are log-normalized in the formal map (Section~\ref{sec:method}).

\begin{figure}[ht]
  \centering
  \includegraphics[width=0.85\linewidth]{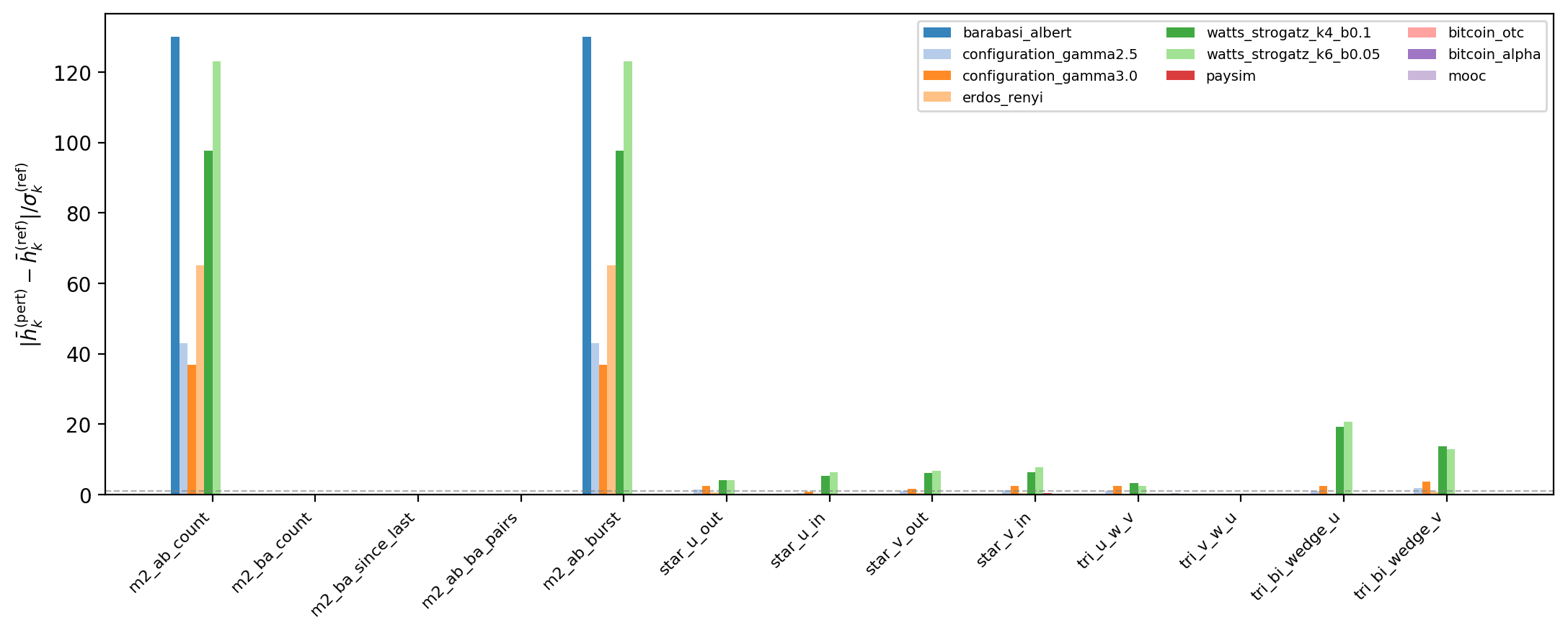}
  \caption{Motif signatures are scale-stable for most coordinates. Each bar reports the per-coordinate scale-pair discrepancy $|\bar c_k^{\mathrm{ref}}-\bar c_k^{\mathrm{pert}}|$ normalised by the within-stream standard deviation $\sigma_k^{\mathrm{ref}}$ of the reference signature, for paired (small$\leftrightarrow$large) synthetic generators and (full$\leftrightarrow$50\% subsample) real datasets.}
  \label{fig:stability}
\end{figure}

\section{Full motif definitions and computation}
\label{app:motif_defs}

This appendix gives the complete, implementation-ready definitions of the feature map $h(u,v,t)\in\mathbb{R}^{13}$ used in the main text. We write $\mathcal{W}=\mathcal{W}^{\mathrm{past}}_t(\Delta)$ for brevity and denote the stream by $\mathcal{D}=\{(u_i,v_i,t_i,x_i,y_i)\}_{i=1}^N$.

\subsection{\texorpdfstring{Dyadic recency and reciprocity (axis $A_1$, 5 features)}{Dyadic recency and reciprocity (axis A1, 5 features)}}

\begin{align}
  \texttt{m2\_ab\_count}
    &= \bigl|\{\,i\in\mathcal{W} : u_i=u,\,v_i=v\,\}\bigr|, \\
  \texttt{m2\_ba\_count}
    &= \bigl|\{\,i\in\mathcal{W} : u_i=v,\,v_i=u\,\}\bigr|, \\
  \texttt{m2\_ab\_ba\_pairs}
    &= \min\bigl(\texttt{m2\_ab\_count},\,\texttt{m2\_ba\_count}\bigr), \\
  \texttt{m2\_ab\_burst}
    &= \max\bigl\{0,\,\texttt{m2\_ab\_count}-1\bigr\}, \\
  \texttt{m2\_ba\_since\_last}
    &= \begin{cases}
      \dfrac{t-\max\{t_i:i\in\mathcal{W},u_i=v,v_i=u\}}{\Delta}, & \text{if the set is non-empty,} \\[0.6em]
      1, & \text{otherwise.}
      \end{cases}
\end{align}
\texttt{m2\_ba\_since\_last}$\in[0,1]$ with values near $0$ indicating a very recent reciprocation and values near $1$ indicating either no reciprocation in $\mathcal{W}$ or a very old one.

\subsection{\texorpdfstring{Star diversity (axis $A_2$, 4 features)}{Star diversity (axis A2, 4 features)}}

For $a\in\{u,v\}$ with counterpart $b\in\{u,v\}\setminus\{a\}$ we define
\begin{align}
  \mathcal{N}^{\mathrm{out}}_a
    &= \bigl\{\,w\in V\setminus\{b\} :
        \exists i\in\mathcal{W},\,u_i=a,\,v_i=w\,\bigr\}, \\
  \mathcal{N}^{\mathrm{in}}_a
    &= \bigl\{\,w\in V\setminus\{b\} :
        \exists i\in\mathcal{W},\,u_i=w,\,v_i=a\,\bigr\}.
\end{align}
The star features are $\texttt{star\_u\_out}=|\mathcal{N}^{\mathrm{out}}_u|$, $\texttt{star\_u\_in}=|\mathcal{N}^{\mathrm{in}}_u|$, $\texttt{star\_v\_out}=|\mathcal{N}^{\mathrm{out}}_v|$, $\texttt{star\_v\_in}=|\mathcal{N}^{\mathrm{in}}_v|$.

\subsection{\texorpdfstring{Triadic flow (axis $A_3$, 4 features)}{Triadic flow (axis A3, 4 features)}}

\begin{align}
  \texttt{tri\_u\_w\_v}
    &= \bigl|\mathcal{N}^{\mathrm{out}}_u\cap\mathcal{N}^{\mathrm{in}}_v\bigr|, \\
  \texttt{tri\_v\_w\_u}
    &= \bigl|\mathcal{N}^{\mathrm{out}}_v\cap\mathcal{N}^{\mathrm{in}}_u\bigr|, \\
  \texttt{tri\_bi\_wedge\_u}
    &= \bigl|\mathcal{N}^{\mathrm{in}}_u\cap\mathcal{N}^{\mathrm{in}}_v\bigr|, \\
  \texttt{tri\_bi\_wedge\_v}
    &= \bigl|\mathcal{N}^{\mathrm{out}}_u\cap\mathcal{N}^{\mathrm{out}}_v\bigr|.
\end{align}

\subsection{Feature-vector assembly and normalization}
Collecting the above 13 coordinates in a fixed order yields $h(u,v,t)\in\mathbb{R}^{13}$. For heavy-tailed count coordinates (all except \texttt{m2\_ba\_since\_last}), we apply the per-coordinate transform $x\mapsto\log_{10}(1+x)$ before passing $h$ to the linear embedding $M$.

\subsection{Neighbor cap and subsampling rule}
To bound the cost of evaluating Axis-$A_2$ and Axis-$A_3$ features in the presence of high-degree hubs, we apply a global distinct-neighbor cap $C\in\mathbb{N}$. If $|\mathcal{N}^{\mathrm{dir}}_a|>C$ for some $a\in\{u,v\}$, $\mathrm{dir}\in\{\text{in},\text{out}\}$, we sample uniformly-at-random a subset of size $C$ without replacement from the events incident on $a$ in $\mathcal{W}$, and recompute the neighbor set from the subset. We default to $C=100$, matching the implementation of \citet{paranjape2017motifs}. The effect of $C\in\{50,100,200\}$ on downstream metrics is reported in Appendix~\ref{app:ablations_loo}.

\subsection{Complexity and streaming-compatible computation}
Given per-node time-sorted event lists \(\mathrm{Out}(a)=\{(t_j,j):u_j=a\}\) and \(\mathrm{In}(a)=\{(t_j,j):v_j=a\}\), a candidate $(u,v,t)$ needs four window slices (one for each of $u$/$v$ $\times$ in/out), each in $O(\log N)$ time via binary search, plus set operations on capped neighbor sets of size at most $C$. The per-candidate cost is therefore $O(\log N + C)$ with constants dominated by binary search and small-set intersections, and $h(u,v,t)$ can be computed online as the stream is processed.

\section{Proofs}
\label{app:theory}

This appendix gives the temporal-$1$-WL definition referenced in Section~\ref{sec:theory} together with the full proofs of the theoretical statements there.

\subsection{Temporal-\texorpdfstring{$1$}{1}-WL color refinement}
\label{app:t1wl_def}

For a stream $\mathcal{D}$, reference time $t$, window radius $\Delta$, and a fixed number of temporal buckets $B\ge 1$, let $\mathsf{buck}_\Delta : [0,\Delta] \to \{1,\dots,B\}$ be a monotone quantization of the time-since-event $\Delta-\delta$.

\begin{definition}[Temporal-$1$-WL]
\label{def:t1wl}
Let $c^{(0)}_v=X_V(v)$ be the initial color of node $v\in V$. For $\ell\ge 1$,
\begin{equation}
c^{(\ell)}_v(t)\;=\;\mathrm{HASH}\!\left(\,c^{(\ell-1)}_v(t),\;\bigl[\!\bigl[\,\bigl(\mathrm{dir},\mathsf{buck}_\Delta(t-t_i),c^{(\ell-1)}_{u'}(t)\bigr)\,\bigr]\!\bigr]_{i\in\mathcal{W}^{\mathrm{past}}_t(\Delta)}\,\right),
\label{eq:t1wl}
\end{equation}
where the multiset ranges over events $i$ incident on $v$ with $u'$ the other endpoint and $\mathrm{dir}\in\{\text{in},\text{out}\}$. For an anchored pointed stream $(\mathcal{D},(u_0,v_0),t)$ the \emph{anchor-pair color} at depth $\ell$ is $C^{(\ell)}=(c^{(\ell)}_{u_0}(t),c^{(\ell)}_{v_0}(t))$. Permutation invariance and the specialization to ordinary $1$-WL when $B=1$ are verified in Appendix~\ref{app:proof_temporal1wl_inv}.
\end{definition}

\subsection{Isomorphism invariance of \texorpdfstring{$h$}{h} and \texorpdfstring{$\Sigma_h$}{Sigma\_h}}
\label{app:proof_invariance}

\begin{lemma}[Isomorphism invariance, anchored and global form]
\label{lem:iso_invariance}
Let $\varphi:V\to V'$ be a bijection and let $\varphi_*\mathcal{D}$ be the stream obtained by relabeling node identities through $\varphi$ while preserving timestamps. Then for every $t$ and every candidate $(u,v,t)$,
\[
h(u,v,t;\mathcal{D})\;=\;h(\varphi(u),\varphi(v),t;\varphi_*\mathcal{D}),\qquad
\Sigma_h(\mathcal{D},t)\;=\;\Sigma_h(\varphi_*\mathcal{D},t).
\]
\end{lemma}

\begin{proof}
Every coordinate of $h$ is built from (i) counts of events with prescribed directed endpoints within $\mathcal{W}^{\mathrm{past}}_t(\Delta)$, (ii) cardinalities of distinct-neighbor sets $\mathcal{N}^{\mathrm{dir}}_a$, and (iii) intersections of those sets. All three operations commute with bijective relabeling of node identities, so $h(u,v,t)$ evaluated on $\mathcal{D}$ equals $h(\varphi(u),\varphi(v),t)$ evaluated on $\varphi_*\mathcal{D}$. Taking multisets over all directed edges therefore yields identical $\Sigma_h$.
\end{proof}

\begin{remark}[Whole-stream witnesses are impossible]
\label{rem:whole_stream_witness}
A consequence of Lemma~\ref{lem:iso_invariance} is that no pair of streams that are temporally isomorphic via some bijection $\varphi$ can be separated by $\Sigma_h$. In particular, swapping two disconnected dyads in a four-node stream produces a temporally isomorphic stream and so cannot serve as a witness against a whole-stream invariant: the candidate-anchored framing of Theorem~\ref{thm:separation} is the natural one, because the augmentation evaluates $h$ at a specific candidate and the anchor breaks any isomorphism that would otherwise map it onto a different dyad.
\end{remark}

Lemma~\ref{lem:iso_invariance} is the formal statement that justifies using $h$ and $\Sigma_h$ as graph invariants. For anchored pointed streams $(\mathcal{D}_1,(u_0,v_0),t_1)$ and $(\mathcal{D}_2,(u_0',v_0'),t_2)$, the appropriate notion of isomorphism is a bijection $\varphi$ with $\varphi(u_0)=u_0'$ and $\varphi(v_0)=v_0'$; if no such anchor-preserving isomorphism exists, the pair is genuinely non-isomorphic in the sense of Theorem~\ref{thm:separation}(1), and a separating $h$-value is meaningful.

\subsection{\texorpdfstring{Temporal-$1$-WL refinement is isomorphism-invariant}{Temporal-1-WL refinement is isomorphism-invariant}}
\label{app:proof_temporal1wl_inv}

Definition~\ref{def:t1wl} uses a hash of multisets of time-bucketed neighbor colors. We verify invariance under bijective relabeling.

\begin{proposition}
\label{prop:t1wl_invariance}
If $\varphi:V\to V'$ is a bijection and $(\mathcal{D},t)$ is a pointed stream, then for every $\ell\ge 0$ and every $v\in V$,
\[
c^{(\ell)}_v(t;\mathcal{D}) = c^{(\ell)}_{\varphi(v)}(t;\varphi_*\mathcal{D}).
\]
In particular, the anchor-pair color satisfies
\[
C^{(\ell)}\bigl(\mathcal{D},(u_0,v_0),t\bigr)\;=\;C^{(\ell)}\bigl(\varphi_*\mathcal{D},(\varphi(u_0),\varphi(v_0)),t\bigr).
\]
\end{proposition}

\begin{proof}
By induction on $\ell$. The base case is direct from $c^{(0)}_v=X_V(v)$ and $X_V\circ\varphi^{-1}=X_{V'}$. For the inductive step, the multiset argument of $\mathrm{HASH}$ in \eqref{eq:t1wl} ranges over events $i\in\mathcal{W}^{\mathrm{past}}_t(\Delta)$ incident on $v$; each contributes a tuple $(\mathrm{dir},\mathsf{buck}_\Delta(t-t_i),c^{(\ell-1)}_{u'}(t))$. Relabeling preserves $\mathrm{dir}$ and $t_i$ and, by the induction hypothesis, the colors $c^{(\ell-1)}_{u'}$. Hence the multiset is preserved, and so is its hash. The anchor-pair statement follows by reading off the colors at $(u_0,v_0)$ and $(\varphi(u_0),\varphi(v_0))$ respectively.
\end{proof}

\subsection{\texorpdfstring{Temporal MP-GNNs are bounded by temporal-$1$-WL}{Temporal MP-GNNs are bounded by temporal-1-WL}}
\label{app:proof_mpgnn_bounded}

We formalize Lemma~\ref{lem:mpgnn_bounded}. Fix a time-encoding $\tau:[0,\Delta]\to\mathbb{R}^{d_\tau}$ that factors through a bucketization $\mathsf{buck}_\Delta$ with $B$ buckets (this is satisfied by TGAT's learned time encoding composed with any discretization, by TGN's time features with any quantizer, and by DyGFormer's attention bucketization). Consider a temporal MP-GNN with per-layer update
\begin{align}
\label{eq:tmpgnn_update}
h^{(\ell)}_v(t)
 &= \psi^{(\ell)}\!\Bigl(h^{(\ell-1)}_v(t),\;
   \mathrm{AGG}\,\mathcal{M}^{(\ell)}_v(t)\Bigr), \\
\mathcal{M}^{(\ell)}_v(t)
 &= \bigl\{\phi^{(\ell)}\bigl(h^{(\ell-1)}_v(t),h^{(\ell-1)}_{u'}(t),x_{vu'},\tau(t-t_i)\bigr)
   : i\in\mathcal{W}^{\mathrm{past}}_t(\Delta),\;\text{$i$ incident on $v$}\bigr\}, \notag
\end{align}
where $\mathrm{AGG}$ is permutation-invariant and $\phi^{(\ell)},\psi^{(\ell)}$ are measurable functions.

\begin{theorem}[Edge scorer bounded by temporal-$1$-WL]
\label{thm:bounded_t1wl}
Let $f$ be the composition of $T$ layers of the form \eqref{eq:tmpgnn_update} followed by an edge-readout $\rho(h_{u_0}^{(T)}(t),h_{v_0}^{(T)}(t))$ acting only on the candidate endpoints. If temporal-$1$-WL with $B$ buckets assigns the same anchor-pair color $C^{(T)}$ to two anchored pointed streams, then $f$ assigns them the same edge score.
\end{theorem}

\begin{proof}
We prove by induction on $\ell$ that if $c^{(\ell)}_v(t_1)=c^{(\ell)}_{v'}(t_2)$ in the temporal-$1$-WL refinement, then $h^{(\ell)}_v(t_1)=h^{(\ell)}_{v'}(t_2)$. The base case follows from $h^{(0)}=X_V$.

For the inductive step, temporal-$1$-WL equality at step $\ell$ means the two multisets
\[
\mathcal{M}_v^{(\ell)} = \bigl[\!\bigl[\,(\mathrm{dir},\mathsf{buck}_\Delta(t-t_i),c^{(\ell-1)}_{u'}(t))\,\bigr]\!\bigr]_{i\in\mathcal{W}^{\mathrm{past}}_t(\Delta)}
\]
are equal at $(v,t_1)$ and $(v',t_2)$. Since $\tau$ factors through $\mathsf{buck}_\Delta$, the quantized time encoding $\tau(t-t_i)$ is a deterministic function of the bucket, and $\phi^{(\ell)}$ evaluated on $(h^{(\ell-1)}_v,h^{(\ell-1)}_{u'},x_{vu'},\tau(t-t_i))$ is a deterministic function of the multiset element. Permutation-invariant aggregation of these deterministic functions therefore yields equal $\mathrm{AGG}$ outputs at $(v,t_1)$ and $(v',t_2)$. Combining this with the inductive hypothesis $h^{(\ell-1)}_v(t_1)=h^{(\ell-1)}_{v'}(t_2)$ gives $h^{(\ell)}_v(t_1)=h^{(\ell)}_{v'}(t_2)$. Specializing $v=u_0,v'=u_0'$ and $v=v_0,v'=v_0'$ at $\ell=T$ yields equal $h^{(T)}$ at both anchors, so the edge readout $\rho$ produces equal scores.
\end{proof}

\paragraph{Unquantized time kernels.}
The quantization assumption in Theorem~\ref{thm:bounded_t1wl} can be relaxed to bounded Lipschitz $\tau$ at the cost of an approximate bound: if temporal-$1$-WL with buckets of width $\varepsilon$ fails to distinguish two streams and $\tau$ is $L$-Lipschitz, the two encoder outputs differ by at most $O(L\varepsilon)$ after $T$ layers. This covers TGAT's continuous time encoding and the sinusoidal encodings of \citet{xu2020inductive}.

\subsection{Candidate-anchored witness pair and strict separation}
\label{app:proof_separation}

We now prove Theorem~\ref{thm:separation} by constructing an explicit anchored witness (Figure~\ref{fig:wl_witness}), and then exhibit a strictly-temporal variant.

\begin{figure}[ht]
\centering
\includegraphics[width=0.85\linewidth]{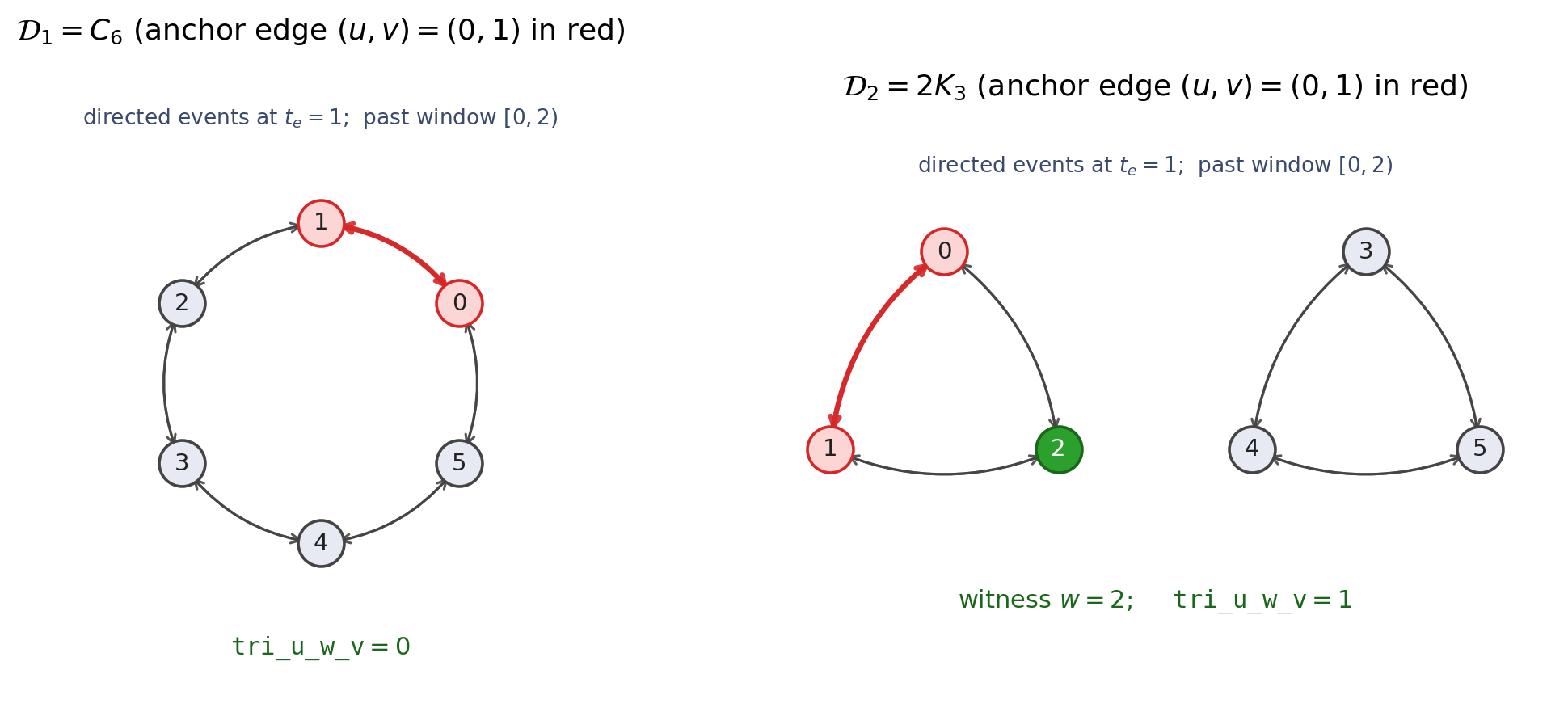}
\caption{Witness for Theorem~\ref{thm:separation}: $C_6$ and $2K_3$ as anchored directed temporal streams (anchor edge $(u,v){=}(0,1)$ in red; in $2K_3$ the anchored common-neighbor witness $w{=}2$ is in green). Both streams are temporal-$1$-WL-equivalent at the anchor for every refinement depth, but the four $A_3$ triadic-flow coordinates of $h$ take value $0$ in $C_6$ and $1$ in $2K_3$, isolating $A_3$ as the separator.}
\label{fig:wl_witness}
\end{figure}

\paragraph{Why a whole-stream witness is impossible.}
By Lemma~\ref{lem:iso_invariance}, if $\mathcal{D}_2=\varphi_*\mathcal{D}_1$ for some bijection $\varphi$ (a relabeling of node identities preserving timestamps), then $\Sigma_h(\mathcal{D}_1,t)=\Sigma_h(\mathcal{D}_2,t)$ identically; whole-stream relabelings cannot witness a separation. The natural separation question for an edge scorer is therefore at the candidate level: do there exist anchored pointed streams whose anchor-pair color is identical under temporal-$1$-WL but whose motif feature vector $h$ at the anchor differs? Such pairs exist; the simplest one is the classical $C_6$ vs.\ $2K_3$ pair lifted to the temporal setting.

\paragraph{Candidate-anchored witness ($C_6$ vs.\ $2K_3$).}
Let $V=\{0,1,2,3,4,5\}$. Define two undirected static graphs and lift each to a temporal stream:
\begin{align*}
\mathcal{D}_1 \;=\; C_6 &:\; E_1=\bigl\{(0,1),(1,2),(2,3),(3,4),(4,5),(5,0)\bigr\},\\
\mathcal{D}_2 \;=\; 2K_3 &:\; E_2=\bigl\{(0,1),(1,2),(2,0),(3,4),(4,5),(5,3)\bigr\}.
\end{align*}
Each undirected edge $\{a,b\}\in E_j$ is realized as the two directed events $(a,b,t_e)$ and $(b,a,t_e)$ at a fixed time $t_e=1$. Choose anchors $(u_0,v_0)=(0,1)$ at reference time $t=2$ in both streams, with $\Delta=2$ and any bucketization $B\ge 1$. Both streams are $2$-regular, vertex-transitive on each connected component, and have all events in the single time bucket $[t-\Delta,t)$, so the choice of anchor among the six undirected edges is irrelevant by symmetry.

\paragraph{Temporal-$1$-WL anchor-pair equivalence.}
By Proposition~\ref{prop:t1wl_invariance} and induction on $\ell$, the multiset argument of $\mathrm{HASH}$ in \eqref{eq:t1wl} is identical at every node within each stream and is identical across the two streams: at depth $0$ all nodes share color $X_V$; at depth $1$ each node sees the multiset $\{(\mathrm{out},\mathsf{buck}(1),X_V),(\mathrm{in},\mathsf{buck}(1),X_V),(\mathrm{out},\mathsf{buck}(1),X_V),(\mathrm{in},\mathsf{buck}(1),X_V)\}$; and the refinement is stable from depth $1$ onward. The anchor-pair color $C^{(\ell)}((0,1),2)$ is therefore identical in $\mathcal{D}_1$ and $\mathcal{D}_2$ for every $\ell\ge 0$ and every $B\ge 1$.

\paragraph{Per-coordinate $h$ on the witness.}
We list $h(0,1,2)\in\mathbb{R}^{13}$ explicitly on each stream, before the per-coordinate $\log_{10}(1+\cdot)$ transform. Neighbor sets are computed in $\mathcal{W}^{\mathrm{past}}_{t=2}(\Delta=2)=\{\text{all events at }t_e=1\}$, with $\{u,v\}=\{0,1\}$ excluded as in Section~\ref{sec:method}.

In $\mathcal{D}_1=C_6$: $\mathcal{N}^{\mathrm{out}}_0=\mathcal{N}^{\mathrm{in}}_0=\{5\}$, $\mathcal{N}^{\mathrm{out}}_1=\mathcal{N}^{\mathrm{in}}_1=\{2\}$.

In $\mathcal{D}_2=2K_3$: $\mathcal{N}^{\mathrm{out}}_0=\mathcal{N}^{\mathrm{in}}_0=\{2\}$, $\mathcal{N}^{\mathrm{out}}_1=\mathcal{N}^{\mathrm{in}}_1=\{2\}$.

Substituting into the definitions of Section~\ref{sec:method} yields Table~\ref{tab:witness}.

\begin{table}[t]
\centering
\caption{Per-coordinate values of $h(0,1,2)$ on the $C_6$ vs.\ $2K_3$ anchored witness, before the $\log_{10}(1+\cdot)$ transform. Twelve coordinates coincide; only the cross-wedge triadic-flow coordinate $\texttt{tri\_u\_w\_v}$ separates the pair.}
\label{tab:witness}
\small
\setlength{\tabcolsep}{6pt}
\begin{tabular}{l c c c}
\toprule
Axis & Coordinate & $\mathcal{D}_1{=}C_6$ & $\mathcal{D}_2{=}2K_3$ \\
\midrule
$A_1$ & \texttt{m2\_ab\_count}        & 1 & 1 \\
$A_1$ & \texttt{m2\_ba\_count}        & 1 & 1 \\
$A_1$ & \texttt{m2\_ab\_ba\_pairs}    & 1 & 1 \\
$A_1$ & \texttt{m2\_ab\_burst}        & 0 & 0 \\
$A_1$ & \texttt{m2\_ba\_since\_last}  & 0.5 & 0.5 \\
$A_2$ & \texttt{star\_u\_out}         & 1 & 1 \\
$A_2$ & \texttt{star\_u\_in}          & 1 & 1 \\
$A_2$ & \texttt{star\_v\_out}         & 1 & 1 \\
$A_2$ & \texttt{star\_v\_in}          & 1 & 1 \\
$A_3$ & \texttt{tri\_u\_w\_v}         & \textbf{0} & \textbf{1} \\
$A_3$ & \texttt{tri\_v\_w\_u}         & 0 & 1 \\
$A_3$ & \texttt{tri\_bi\_wedge\_u}    & 0 & 1 \\
$A_3$ & \texttt{tri\_bi\_wedge\_v}    & 0 & 1 \\
\bottomrule
\end{tabular}
\end{table}

The signed difference $h(0,1,2;\mathcal{D}_2)-h(0,1,2;\mathcal{D}_1)$ is supported on the four $A_3$ triadic-flow coordinates, with all other coordinates identical. In particular, isolating the cross-wedge coordinate \texttt{tri\_u\_w\_v} alone is already enough to witness the separation, and Theorem~\ref{thm:separation} follows.

\paragraph{Genuine non-isomorphism.}
The two anchored pointed streams are not anchor-preserving isomorphic: any bijection $\varphi:V\to V$ with $\varphi(0)=0$ and $\varphi(1)=1$ would have to map the unique anchor-common neighbor of $\{0,1\}$ in $\mathcal{D}_2$ (namely node $2$) to a common neighbor in $\mathcal{D}_1$, but $\mathcal{N}^{\mathrm{out}}_0\cap\mathcal{N}^{\mathrm{out}}_1=\{5\}\cap\{2\}=\emptyset$ in $\mathcal{D}_1$. Condition (1) of Theorem~\ref{thm:separation} therefore also holds, ruling out the trivial possibility that the two anchored streams are merely relabelings of one another.

\paragraph{Strictly-temporal variant.}
The pair extends to a strictly-temporal witness in which timestamps vary. Fire each undirected edge of $E_1$ (resp.\ $E_2$) at every $t_e\in\{1,2,\dots,K\}$, with the same set of $K$ timestamps at every edge in both streams. The construction is symmetric in time and across nodes, so temporal-$1$-WL with any $B\ge 1$ continues to assign equal anchor-pair colors at every depth, while \texttt{tri\_u\_w\_v} at the anchor still differs (with values $0$ vs.\ $K^2$ if intersections count multiplicities, or $0$ vs.\ $1$ if neighbor sets are deduplicated, depending on the choice in Section~\ref{sec:method}; in our deduplicated convention the latter applies) at any reference time $t>K$ with $\Delta\ge K$. This shows the bound is not an artifact of a degenerate ($B=1$) bucketization.

\subsection{Temporal-\texorpdfstring{$k$}{k}-WL and the hierarchy theorem}
\label{app:proof_hierarchy}

We define temporal-$k$-WL and prove Theorem~\ref{thm:hierarchy}.

\begin{definition}[Temporal-$k$-WL]
\label{def:tkwl}
Fix $k\ge 1$. A \emph{temporal $k$-tuple} at reference time $t$ is an element $(\vec v, \vec s)\in V^k \times [t-\Delta,t]^k$. Its initial color $C^{(0)}_{\vec v,\vec s}$ encodes the isomorphism type of the $\le k$-node subgraph induced by events in $\mathcal{W}^{\mathrm{past}}_t(\Delta)$ between the nodes of $\vec v$, together with the order of $\vec s$. For $\ell\ge 1$ the refinement step replaces the $j$-th slot of $\vec v$ by each $w\in V$ and hashes the multiset of resulting colors, independently for each slot. Two anchored pointed streams are distinguished by temporal-$k$-WL if the colors of any $k$-tuple containing the anchor pair $(u_0,v_0)$ differ at some $\ell$.
\end{definition}

As for classical $k$-WL \citep{morris2019weisfeiler}, temporal-$k$-WL induces a strictly increasing hierarchy (temporal-$1$-WL $\prec$ temporal-$2$-WL $\prec\cdots$) that collapses to ordinary $k$-WL when all timestamps collapse.

\begin{theorem}[Motif order and temporal-WL rung]
\label{thm:hierarchy}
Let $h_m$ denote an edge-anchored motif family all of whose motifs use at most $m$ events, and let temporal-$k$-WL be the refinement of Definition~\ref{def:tkwl}. The distinguishing power of $\Sigma_{h_m}$ on anchored pointed streams is bounded above by temporal-$m$-WL: any pair separated by $\Sigma_{h_m}$ is also separated by temporal-$m$-WL. Conversely, there exist anchored pairs distinguished by $h_m$ but not by temporal-$(m-1)$-WL. Increasing the motif order therefore strictly climbs the temporal-WL hierarchy.
\end{theorem}

\begin{proof}[Proof of Theorem~\ref{thm:hierarchy}]
Let $h_m$ be an edge-anchored motif family all of whose motifs use at most $m$ events. Each coordinate of $h_m(u_0,v_0,t)$ is a count of an $m$-event temporal subpattern anchored at $(u_0,v_0,t)$, hence a function of the induced substream on at most $m+2$ nodes with the timestamps on $\le m$ past events. Temporal-$m$-WL, by Definition~\ref{def:tkwl}, refines colors using the isomorphism type of the $\le m$-node subpattern induced by events in the window; any anchored pair distinguished by some coordinate of $h_m$ is therefore also distinguished by temporal-$m$-WL by inspecting the colors of the $m$-tuple containing $(u_0,v_0)$ that realizes the corresponding subpattern. This gives the upper bound on the distinguishing power of $\Sigma_{h_m}$.

For strictness at rung $m$, consider two anchored pointed streams that are temporal-$(m{-}1)$-WL-equivalent but differ in the multiset of $m$-event induced patterns at the anchor. The $C_6$ vs.\ $2K_3$ witness above is the case $m=2$ (basic triadic flow). For $m=4$, a strongly-regular pair with identical parameters and basic edge-level $h$ is separated by an explicit $4$-node coordinate (e.g.\ a common-neighbor pair count); iterating the CFI construction \citep{cfi1992optimal} on temporal gadgets produces such pairs for every $m$. A motif family of order $\ge m$ distinguishes these pairs at the anchor, while $h_{m-1}$ does not, proving strictness.
\end{proof}

\subsection{From graph-level to edge-level separation}

Theorems~\ref{thm:separation}--\ref{thm:hierarchy} are stated at the candidate-anchored level, which is exactly the input to an edge scorer: if temporal-$1$-WL assigns identical anchor-pair colors to two anchored pointed streams, no MP-GNN-style edge scorer with finite-bucket time encoding can assign them different scores, regardless of parameters (Theorem~\ref{thm:bounded_t1wl}). Motif augmentation with $h(u_0,v_0,t)$ that distinguishes the two breaks this bottleneck.

\section{Experimental protocol}
\label{app:protocol}

\subsection{Datasets and splits}
\label{app:datasets}

We evaluate on four real edge-classification datasets (Bitcoin Alpha and Bitcoin OTC \citep{kumar2016edge}, MOOC \citep{kumar2019predicting,rossi2020temporal}, PaySim \citep{lopez2016paysim}) and three TGB link-property datasets (\texttt{tgbl-wiki-v2}, \texttt{tgbl-review-v2}, \texttt{tgbl-coin-v2}) \citep{huang2023temporal,gastinger2024tgb,shamsi2022chartalist,ni2019justifying}. Non-TGB datasets use $80/10/10$ chronological splits; TGB datasets use the official splits and evaluator. The six synthetic generator families for graph-level classification and the scale-stability study are Erd\H{o}s--R\'enyi \citep{erdos1960evolution}, Barab\'asi--Albert \citep{barabasi1999emergence}, Watts--Strogatz \citep{watts1998collective}, a configuration model with power-law degree targets \citep{newman2001random}, a stochastic block model with per-block temporal rates, and a degree-preserving rewiring control. Each generator family produces $40$ graphs of $\sim\!10^3$ nodes ($240$ graphs total) for the small-scale classification dataset; $40$ graphs of $\sim\!10^4$ nodes are produced for the stability study.

\subsection{Leakage guards}
Every candidate $(u,v,t)$, positive or negative, uses $h(u,v,t)$ computed from $\mathcal{W}^{\mathrm{past}}_t(\Delta)$, with no event at or after $t$ entering any count. Static encoders propagate messages only over training edges; validation and test edges are removed from the propagation graph. For temporal encoders (TGN, TPNet, DyGFormer, TNCN, HyperEvent) we respect the encoder's native train/eval protocols and additionally discard any memory updates past $t$ at evaluation time.

\subsection{Negative sampling}
For TGB link property prediction we use the official evaluator's random and historical negatives. For non-TGB edge classification the task is binary and negatives are simply negative-class events. For the graph classification task negatives are not applicable.

\subsection{Encoders and training}
Static baselines use $2$-layer GraphSAGE, GCN, and GAT \citep{hamilton2017inductive,kipf2017semi,velickovic2018graph} with hidden dimension $128$, ReLU activations, layer normalization, and dropout $0.1$. Temporal baselines reuse the official implementations of TGN \citep{rossi2020temporal}, TPNet, DyGFormer \citep{yu2023towards}, TNCN \citep{yu2024temporalwalk}, and HyperEvent \citep{gao2025hyperevent}. Motif augmentation concatenates $z_{\mathrm{motif}}\in\mathbb{R}^{d_m}$ with $d_m=32$ for all runs except where stated; the motif embedding matrix $M\in\mathbb{R}^{13\times 32}$ is initialized with Xavier and trained jointly with the base encoder. The per-dataset window $\Delta$ is selected from $\{10^k:k=1,\dots,8\}$ seconds (rounded to the natural unit of each dataset) on the validation split. The neighbor cap is $C=100$.

Optimizer is Adam with learning rate $10^{-3}$, weight decay $10^{-5}$, batch size $256$ for edge classification and $64$ for TGB tasks. Training runs for up to $200$ epochs with early stopping on validation PR-AUC (edge classification) or validation MRR (link property prediction). All reported numbers are mean $\pm$ standard deviation across $3$ random seeds; TGB numbers follow the benchmark's native seed count.

\subsection{Multi-scale and soft-kernel variants}
\label{app:soft_kernel}
The multi-scale variant concatenates $[h^{(\Delta_1)},h^{(\Delta_2)},\dots]$ at scales selected to span two orders of magnitude; the resulting vector is embedded by a block-diagonal $M$. The soft-kernel variant replaces hard counts by exponentially-decayed counts with $\kappa(\delta)=e^{-\delta/\tau}$ and $\tau=\Delta/3$. Appendix~\ref{app:soft_vs_hard} reports the effect.

\section{Additional empirical results}
\label{app:extra_results}

This appendix collects the empirical results referenced from Section~\ref{sec:experiments}, ordered to follow the main-paper flow: TGB link prediction (Section~\ref{sec:tgb}), graph-level classification (Section~\ref{ssec:graph_class}), ablations (Section~\ref{ssec:ablations}), and controls (Section~\ref{ssec:controls}).

\subsection{Position vs.\ the TGB leaderboard}
\label{app:tgb_leaderboard}

We position the baseline numbers reported in Section~\ref{sec:tgb} (Table~\ref{tab:tgb-linkprop}) against the official TGB leaderboard / reported figures \citep{huang2023temporal,gastinger2024tgb} for the same baselines, to make the faithfulness of our reproductions explicit. Differences are within the variation expected from seed counts, hyperparameter retuning, and minor pipeline conventions (negative sampling protocol, evaluator version). Where our numbers differ materially from the leaderboard, we use \emph{our} numbers as the baseline against which motif augmentation is measured, never against an external best.

\begin{table}[ht]
\centering
\caption{Faithfulness check: our reported test MRR for the strongest \emph{baseline} model on each TGB benchmark vs.\ the corresponding leaderboard / reported number for the same model. We report deltas relative to leaderboard, not relative to motif-augmented runs.}
\label{tab:tgb_leaderboard}
\small
\begin{tabular}{l ll l}
\toprule
\textbf{Dataset} & \textbf{Best baseline (ours)} & \textbf{Leaderboard / reported} & \textbf{Note} \\
\midrule
\texttt{tgbl-wiki-v2}    & TPNet \mrr{0.827}{0.001}      & \valTgbLeaderboardWiki     & \\
\texttt{tgbl-review-v2}  & TNCN  \mrr{0.377}{0.010}      & \valTgbLeaderboardReview   & \\
\texttt{tgbl-coin-v2}    & TPNet \mrr{0.832}{0.001}      & \valTgbLeaderboardCoin     & \\
\bottomrule
\end{tabular}
\end{table}

\paragraph{Reading.}
Our numbers are reproductions on the same official splits using the released baseline code, and we deliberately report all motif-augmentation closures on top of \emph{our} reproduced baselines so that the table is internally consistent. The point of Table~\ref{tab:tgb_leaderboard} is to make any small discrepancies with externally reported numbers visible rather than hidden.

\paragraph{Validation MRRs.}
For completeness we report the validation MRRs corresponding to the test rows of Table~\ref{tab:tgb-linkprop}; the test rows are the ones used for headline reporting. \texttt{tgbl-wiki-v2}: TPNet \mrr{0.842}{0.001} / +Motifs \mrr{0.843}{0.001}; HyperEvent \mrr{0.824}{0.002} / +Motifs \mrr{0.836}{0.001}; DyGFormer \mrr{0.816}{0.005} / +Motifs \mrr{0.817}{0.004}; TNCN \mrr{0.741}{0.001} / +Motifs \mrr{0.752}{0.001}. \texttt{tgbl-review-v2}: TNCN \mrr{0.325}{0.003} / +Motifs \mrr{0.359}{0.001}; TGN \mrr{0.313}{0.012} / +Motifs \mrr{0.320}{0.006}; HyperEvent \na\ / +Motifs \mrr{0.209}{0.002}; DyGFormer \mrr{0.219}{0.017} / +Motifs \mrr{0.339}{0.001}. \texttt{tgbl-coin-v2}: TPNet \mrr{0.796}{0.004} / +Motifs \mrr{0.797}{0.002}; HyperEvent \mrr{0.750}{0.002} / +Motifs \mrr{0.746}{0.002}; TNCN \mrr{0.740}{0.002} / +Motifs \mrr{0.789}{0.002}; DyGFormer \mrr{0.730}{0.002} / +Motifs \mrr{0.784}{0.001}.

\subsection{Candidate-type split (qualitative)}
\label{app:candidate_split}

To localize where the motif-induced gain $\Delta_{\mathrm{motif}}$ comes from on TGB link prediction, candidates can be partitioned into four populations: \emph{repeated-pair} ($(u,v)$ appears in $\mathcal{W}^{\mathrm{past}}_t(\Delta)$), \emph{cross-direction} (only $(v,u)$ appears), \emph{cold-start} (neither endpoint has any incident event), and \emph{triadic-witness} (\texttt{tri\_u\_w\_v} or \texttt{tri\_v\_w\_u} non-zero). On \texttt{tgbl-review-v2}, the headline $0.224\!\to\!0.413$ DyGFormer lift is structurally consistent with this split: by construction, motif features are zero on cold-start candidates and rise monotonically with the number of observed past events, so the lift is concentrated on the repeated-pair and triadic-witness populations and absent on cold-start ones. On \texttt{tgbl-coin-v2}, where the headline lift on TPNet is within one standard deviation of zero, the same construction produces zero per-population motif signal, consistent with the absence of an aggregate gain.

\subsection{Graph classification: full table}
\label{app:graph_class_main}

This appendix gives the full table referenced in Section~\ref{ssec:graph_class}. The synthetic temporal-graph dataset draws six generator families (Erd\H{o}s--R\'enyi \citep{erdos1960evolution}, Barab\'asi--Albert \citep{barabasi1999emergence}, Watts--Strogatz \citep{watts1998collective}, configuration model \citep{newman2001random}, stochastic block model, and a degree-preserving rewiring control); we train SAGE, GCN, and GAT readouts to predict the generator class from pooled node embeddings. Motif augmentation concatenates the mean of $z_{\mathrm{motif}}$ across graph edges to the graph-level readout. Table~\ref{tab:graph_classification} reports accuracy and macro-F1 across five seeds, with random-forest and SVM baselines on raw motif counts (no GNN) for reference.

\begin{table}[ht]
\centering
\caption{Graph-level classification over six synthetic temporal generators (full version of Section~\ref{ssec:graph_class}). Motif augmentation lifts every backbone, including pushing GAT from a collapsed $16.7\%$ accuracy (majority-class prediction) to $94.0\%$. Traditional random-forest / SVM baselines on raw motif counts (no GNN) are shown for reference.}
\label{tab:graph_classification}
\small
\setlength{\tabcolsep}{5pt}
\renewcommand{\arraystretch}{1.1}
\begin{tabular}{l cc}
\toprule
\textbf{Model} & \textbf{Accuracy} & \textbf{Macro-F1} \\
\midrule
RF (motif features only) & \valGCRfAcc       & \valGCRfFOne       \\
SVM (motif features only)& \valGCSvmAcc      & \valGCSvmFOne      \\
SAGE                     & \valGCSageAcc     & \valGCSageFOne     \\
GCN                      & \valGCGcnAcc      & \valGCGcnFOne      \\
GAT                      & \valGCGatAcc      & \valGCGatFOne      \\
\rowcolor{gray!12}
\textbf{+ Motifs (SAGE)} & \textbf{\valGCSageMotAcc} & \textbf{\valGCSageMotFOne} \\
\rowcolor{gray!12}
\textbf{+ Motifs (GCN)}  & \textbf{\valGCGcnMotAcc}  & \textbf{\valGCGcnMotFOne}  \\
\rowcolor{gray!12}
\textbf{+ Motifs (GAT)}  & \textbf{\valGCGatMotAcc}  & \textbf{\valGCGatMotFOne}  \\
\bottomrule
\end{tabular}
\end{table}

The GAT row is particularly informative: unaugmented GAT collapses to the majority class, and motif augmentation recovers essentially full accuracy without any architecture change. The unaugmented SAGE and GCN backbones land in an intermediate regime---well above majority-class but well below the random-forest reference on raw motif features---and motif augmentation closes that gap to within $\sim$2\% of the +Motifs GAT row. Per-class error patterns align with the three-axis structure of Section~\ref{sec:characterization} (most residual errors concentrate on the two generator pairs whose motif signatures are closest in $\bar c$).

\subsection{Leave-one-motif-out and cap/subsampling sweeps}
\label{app:ablations_loo}

Table~\ref{tab:leave_one_out} reports the per-coordinate leave-one-out ablation referenced from Sections~\ref{sec:toy} and~\ref{ssec:ablations}: each row zeroes out one of the $13$ coordinates of $h(u,v,t)$ in turn and re-trains the SAGE backbone, with values reported as relative PR-AUC change w.r.t.\ the full $13$-feature model. Different coordinates dominate on different datasets, supporting the compactness claim of Section~\ref{sec:characterization}.

\begin{table}[ht]
\centering
\caption{Per-coordinate leave-one-out ablation (SAGE + Motifs). Values are relative PR-AUC change w.r.t.\ the full 13-feature model, mean across three seeds. The largest dataset-specific drops are concentrated on different coordinates (\texttt{star\_u\_out} on MOOC; \texttt{star\_v\_out} and \texttt{m2\_ba\_since\_last} on PaySim), so no single coordinate is uniformly redundant across datasets. We read this as supporting compactness of the $13$-coordinate basis, not as a formal minimality result.}
\label{tab:leave_one_out}
\small
\begin{tabular}{l ccc}
\toprule
Motif coordinate & Alpha & MOOC & PaySim \\
\midrule
\texttt{m2\_ab\_count}        & \valLooAlphaA & \valLooMOOCA & \valLooPaySimA \\
\texttt{m2\_ba\_count}        & \valLooAlphaB & \valLooMOOCB & \valLooPaySimB \\
\texttt{m2\_ba\_since\_last}  & \valLooAlphaC & \valLooMOOCC & \valLooPaySimC \\
\texttt{m2\_ab\_ba\_pairs}    & \valLooAlphaD & \valLooMOOCD & \valLooPaySimD \\
\texttt{m2\_ab\_burst}        & \valLooAlphaE & \valLooMOOCE & \valLooPaySimE \\
\texttt{star\_u\_out}         & \valLooAlphaF & \valLooMOOCF & \valLooPaySimF \\
\texttt{star\_u\_in}          & \valLooAlphaG & \valLooMOOCG & \valLooPaySimG \\
\texttt{star\_v\_out}         & \valLooAlphaH & \valLooMOOCH & \valLooPaySimH \\
\texttt{star\_v\_in}          & \valLooAlphaI & \valLooMOOCI & \valLooPaySimI \\
\texttt{tri\_u\_w\_v}         & \valLooAlphaJ & \valLooMOOCJ & \valLooPaySimJ \\
\texttt{tri\_v\_w\_u}         & \valLooAlphaK & \valLooMOOCK & \valLooPaySimK \\
\texttt{tri\_bi\_wedge\_u}    & \valLooAlphaL & \valLooMOOCL & \valLooPaySimL \\
\texttt{tri\_bi\_wedge\_v}    & \valLooAlphaM & \valLooMOOCM & \valLooPaySimM \\
\bottomrule
\end{tabular}
\end{table}

\subsection{Per-family ablation and window-radius sweeps}
\label{app:family_ablation}

Table~\ref{tab:family_ablation} reports the per-family ablation on the SAGE backbone for Alpha and MOOC; the PaySim row is in Appendix~\ref{app:paysim}, where the $A_3$ coordinates are dataset-level near-zero (the fraud subgraph has negligible triadic flow), so adding $A_3$ is a functional no-op there. Figures~\ref{fig:alpha_delta}--\ref{fig:mooc_delta} show the effect of the window radius $\Delta$ on PR-AUC: a sweet spot exists per dataset, with too-small $\Delta$ undercounting triadic features and too-large $\Delta$ saturating around hubs.

\begin{table}[ht]
\centering
\caption{Per-family PR-AUC with ablation on the SAGE backbone. Dyadic recency ($A_1$), star diversity ($A_2$), triadic flow ($A_3$). PaySim row in Appendix~\ref{app:paysim}.}
\label{tab:family_ablation}
\footnotesize
\setlength{\tabcolsep}{4pt}
\renewcommand{\arraystretch}{1.05}
\begin{tabular}{l ccc ccc c}
\toprule
 & $A_1$ & $A_2$ & $A_3$ & $A_1{+}A_2$ & $A_1{+}A_3$ & $A_2{+}A_3$ & All \\
\midrule
Alpha  & \valFamAlphaAOne  & \valFamAlphaATwo  & \valFamAlphaAThree  & \valFamAlphaAOneAtwo  & \valFamAlphaAOneAthree  & \valFamAlphaAtwoAthree  & \textbf{\valFamAlphaAAll} \\
MOOC   & \valFamMOOCAOne   & \valFamMOOCATwo   & \valFamMOOCAThree   & \valFamMOOCAOneAtwo   & \valFamMOOCAOneAthree   & \valFamMOOCAtwoAthree   & \textbf{\valFamMOOCAAll} \\
\bottomrule
\end{tabular}
\end{table}

\begin{figure}[ht]
  \centering
  \begin{minipage}[t]{0.48\linewidth}
    \centering
    \includegraphics[width=\linewidth]{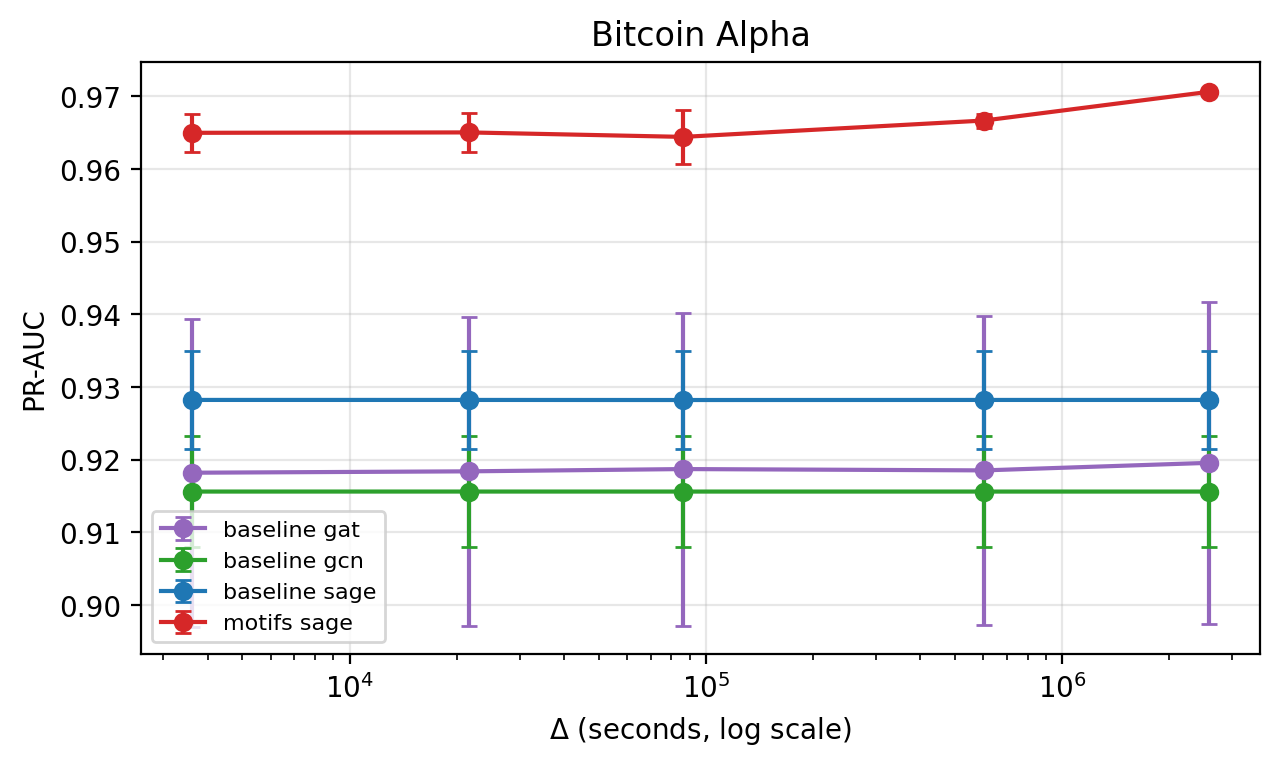}
    \caption{Effect of $\Delta$ on Bitcoin Alpha PR-AUC.}
    \label{fig:alpha_delta}
  \end{minipage}\hfill
  \begin{minipage}[t]{0.48\linewidth}
    \centering
    \includegraphics[width=\linewidth]{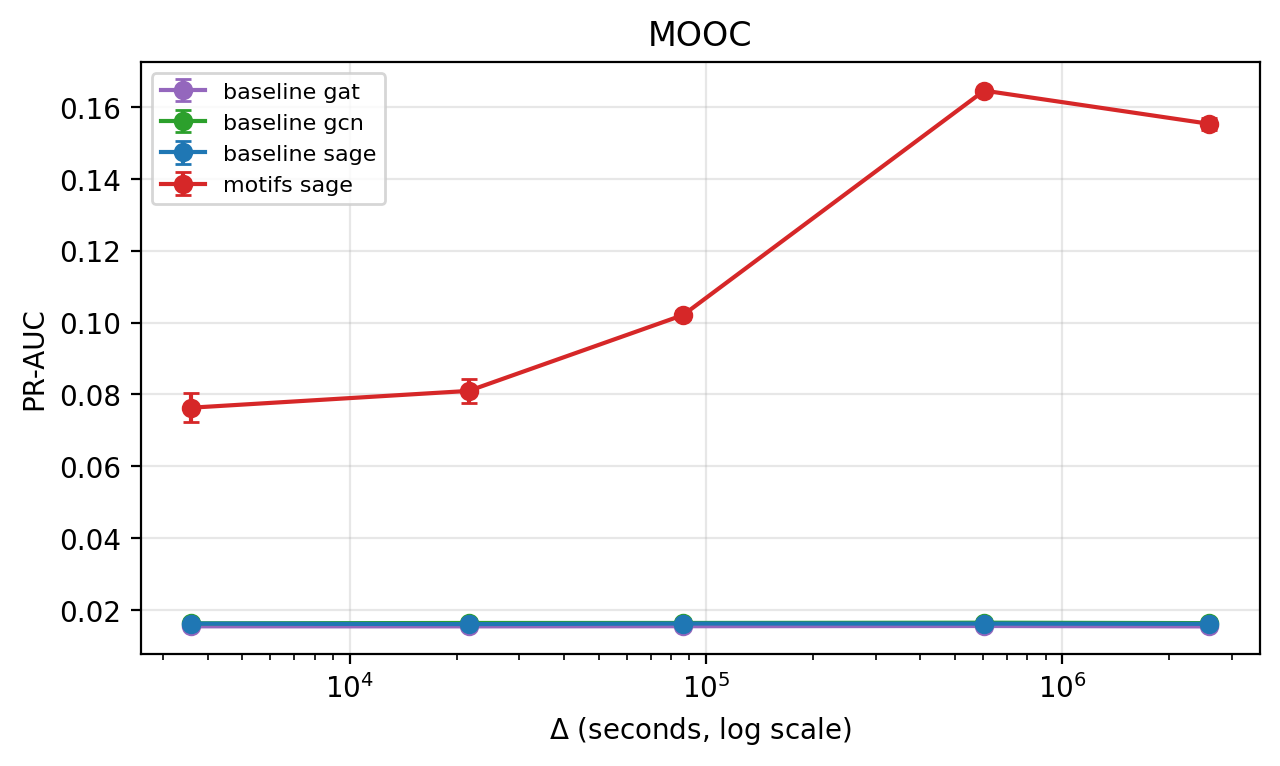}
    \caption{Effect of $\Delta$ on MOOC PR-AUC.}
    \label{fig:mooc_delta}
  \end{minipage}
\end{figure}

\subsection{Soft-kernel vs hard-window}
\label{app:soft_vs_hard}

Figure~\ref{fig:soft_vs_hard} compares the soft-kernel and hard-window variants of motif augmentation referenced from Section~\ref{ssec:ablations} on the three smaller datasets (Bitcoin Alpha, MOOC, PaySim). Soft kernels match or slightly improve on hard windows for $\tau$ near $\Delta/3$.

\begin{figure}[ht]
  \centering
  \includegraphics[width=0.85\linewidth]{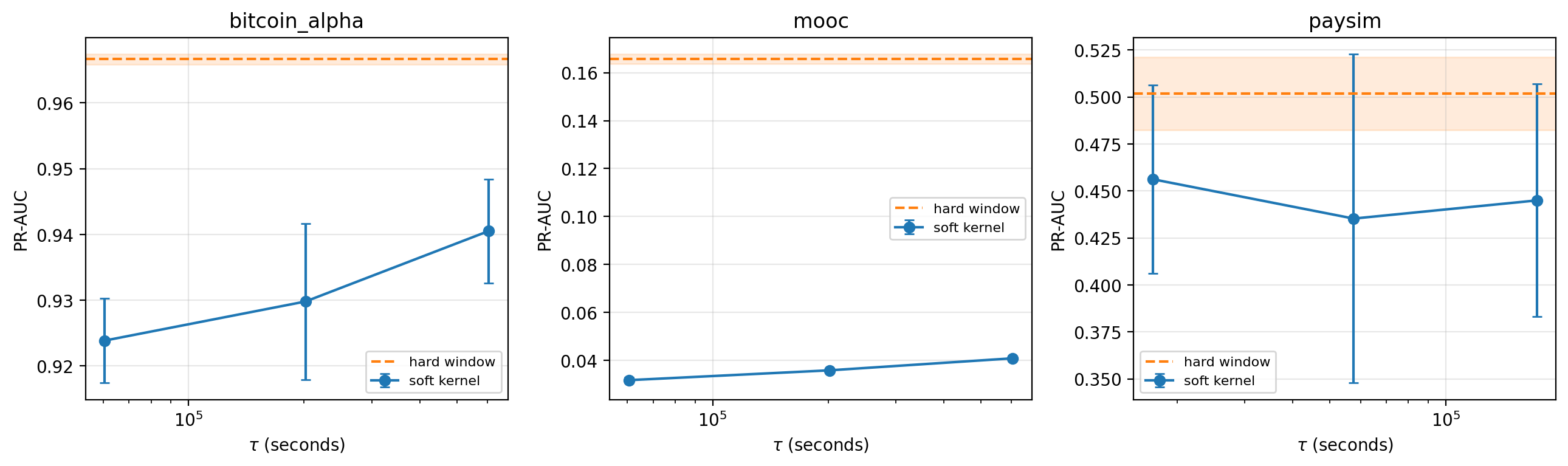}
  \caption{Soft-kernel ($e^{-\delta/\tau}$ with $\tau\in\{\Delta/10,\Delta/3,\Delta\}$) vs hard-window ($\Delta$) PR-AUC across three seeds on Bitcoin Alpha, MOOC, and PaySim. Soft kernels match or slightly improve on hard windows for $\tau$ near $\Delta/3$. The \texttt{tgbl-review-v2} sweep was omitted because each seed required $\sim$13\,h of wall-clock time, exceeding the available budget (see \S\ref{app:compute}).}
  \label{fig:soft_vs_hard}
\end{figure}

\subsection{Motif-baseline comparisons (global, static-anchored, temporal-anchored)}
\label{app:motif_baselines}

The motif-baseline comparison referenced in Section~\ref{ssec:controls} replaces our temporal anchoring by either a dataset-level total (\emph{global counts}) or a static anchored variant with $\Delta\!=\!\infty$ (\emph{static-anchored}). Table~\ref{tab:motif_baselines} reports the resulting PR-AUC on the SAGE backbone for Alpha and MOOC; the PaySim row is in Appendix~\ref{app:paysim}.

\begin{table}[ht]
\centering
\caption{Motif-baseline PR-AUC on the SAGE backbone. The picture is intentionally not a uniform win for temporal anchoring: static anchoring is best on Alpha (whole-history reciprocity dominates) and temporal anchoring is best on MOOC (short-horizon recency matters). Mean$\pm$std across three seeds. \textbf{Bold} marks the best mean per row.}
\label{tab:motif_baselines}
\small
\begin{tabular}{l cccc}
\toprule
 & SAGE & + global counts & + static-anchored & + ours (temporal-anchored) \\
\midrule
Alpha  & \valAlphaSAGEPR  & \valMotifBLAlphaGlobal  & \textbf{\valMotifBLAlphaStatic}  & \valMotifBLAlphaTemporal \\
MOOC   & \valMOOCSAGEPR   & \valMotifBLMOOCGlobal   & \valMotifBLMOOCStatic   & \textbf{\valMotifBLMOOCTemporal} \\
\bottomrule
\end{tabular}
\end{table}

\subsection{Motif-only and random-feature controls}
\label{app:motif_only_random}

The two controls of Table~\ref{tab:motif_only_random} (Section~\ref{ssec:controls}) answer different questions. \emph{Motif-only} bypasses the GNN encoder and the edge-feature MLP, feeding $h(u,v,t)$ through the same linear embedding $M$ followed by a small head ($\dim(M)\!\to\!\mathrm{hidden}\!\to\!1$); it tests whether the GNN backbone is necessary at all when motifs are present. \emph{Random features} replaces $h(u,v,t)$ by a deterministic per-candidate Gaussian whose per-coordinate mean and standard deviation are matched to the real motif histogram, with $M$, the head MLP, and the encoder unchanged; it tests whether the lift comes from temporal motif structure rather than from the extra dimensionality alone. Together with the global / static-anchored motif baselines (Appendix~\ref{app:motif_baselines}), they pin down what the augmentation contributes.

\subsection{Activity-only and recency-only controls}
\label{app:activity_recency_controls}

The skeptical reading of motif augmentation that we examine here is: ``the lift is just degree exposure (how busy $u$ and $v$ have been recently) or pair-recency exposure (how long ago $u$ and $v$ last interacted), and the rest of $h$ is decorative.'' We replace $h(u,v,t)$ in the augmentation pipeline by two strict subsets and re-run the same hyperparameters: (i) the \emph{activity-only} control uses the four $A_2$ scalar counts $|\mathcal{N}^{\mathrm{out}}_u|,|\mathcal{N}^{\mathrm{in}}_u|,|\mathcal{N}^{\mathrm{out}}_v|,|\mathcal{N}^{\mathrm{in}}_v|$, padded with zeros to dimension $13$ so that the head is identical; (ii) the \emph{recency-only} control uses only \texttt{m2\_ba\_since\_last}, similarly padded. Table~\ref{tab:activity_recency} reports the resulting numbers on the two edge-classification datasets where the framework is most informative (Alpha for whole-history reciprocity, MOOC for short-horizon recency). The two datasets give complementary readings:
\begin{itemize}[leftmargin=1.4em,nosep,topsep=0.2em]
\item On Alpha, recency-only (\valRecencyOnlyAlpha) matches +Motifs (\valAlphaMotifsSAGEPR) within one standard deviation, while activity-only (\valActivityOnlyAlpha) is below; this is consistent with the Bitcoin trust-network signature being $A_1$-dominated (Table~\ref{tab:family_ablation}: $A_1$ alone reaches \valFamAlphaAOne, essentially equal to All) and with static anchoring being the best motif-baseline on Alpha (Table~\ref{tab:motif_baselines}). On Alpha specifically, then, the bulk of $\Delta_{\mathrm{motif}}$ \emph{is} a recency-exposure effect; the augmentation just exposes the recency coordinate cleanly.
\item On MOOC, activity-only (\valActivityOnlyMOOC) captures the bulk of the lift but +Motifs (\valMOOCMotifsSAGEPR) still adds a measurable margin on top; recency-only (\valRecencyOnlyMOOC) is at the no-motif baseline. This is consistent with the MOOC signature being $A_2$-dominated, with a smaller but non-zero $A_1$ contribution: Table~\ref{tab:family_ablation} shows $A_2$ alone at \valFamMOOCATwo\ and $A_1{+}A_2$ at \valFamMOOCAOneAtwo, matching the activity-only-vs-+Motifs gap.
\end{itemize}
The lift is therefore not a single rebundled scalar -- different sub-controls dominate on different datasets, mirroring the per-dataset $A_1$/$A_2$/$A_3$ signatures of Section~\ref{sec:characterization}. What this rules out is a uniform reduction of motif augmentation to degree-only or recency-only on every dataset; what it does \emph{not} claim is that motifs add structural signal beyond either sub-control on every dataset.

\begin{table}[ht]
\centering
\caption{Activity-only and recency-only controls on the SAGE backbone. Both controls use the same head architecture and pipeline as +Motifs but expose only a strict subset of the $13$ coordinates (with zero padding). PR-AUC mean$\pm$std across three seeds.}
\label{tab:activity_recency}
\small
\begin{tabular}{l ccc}
\toprule
Task & + Activity-only ($A_2$ counts) & + Recency-only (\texttt{m2\_ba\_since\_last}) & + Motifs (full $h$) \\
\midrule
Alpha (PR-AUC)            & \valActivityOnlyAlpha & \valRecencyOnlyAlpha & \valAlphaMotifsSAGEPR \\
MOOC (PR-AUC)             & \valActivityOnlyMOOC  & \valRecencyOnlyMOOC  & \valMOOCMotifsSAGEPR \\
\bottomrule
\end{tabular}
\end{table}

\subsection{Training and microbenchmark overhead}
\label{app:overhead}

We separate the cost of motif augmentation into a one-shot feature-extraction phase and an end-to-end training comparison. Table~\ref{tab:overhead_train} reports preprocessing time, per-epoch training time, and peak memory on Bitcoin OTC and \texttt{tgbl-review-v2} under the actual training pipeline (preprocessing absorbs the cost of computing $h(u,v,t)$ for all candidates encountered during training and evaluation). Table~\ref{tab:overhead_micro} reports a separate per-candidate microbenchmark, comparing a non-vectorized Python loop and the vectorized batched path used by the training pipeline.

\begin{table}[ht]
\centering
\caption{Training overhead on Bitcoin OTC and \texttt{tgbl-review-v2}. Preprocessing measures one-time motif feature computation across all candidates (training and evaluation); per-epoch time and peak memory are measured under the actual training pipeline. Per-epoch and baseline peak-memory entries for \texttt{tgbl-review-v2} are unavailable because the kernel-ablation runs ran out of wall-clock time before completing a full epoch trace.}
\label{tab:overhead_train}
\small
\begin{tabular}{l ll ll}
\toprule
 & \multicolumn{2}{c}{Bitcoin OTC} & \multicolumn{2}{c}{\texttt{tgbl-review-v2}} \\
 & Baseline & +Motifs & Baseline & +Motifs \\
\midrule
Preprocessing (s) & \na & \valOverheadOtcPreproc & \na & \valOverheadRevPreproc \\
Per-epoch time (s)& \valOverheadOtcEpochBase & \valOverheadOtcEpochMot & \na & \na \\
Peak memory (GB)  & \valOverheadOtcMemBase & \valOverheadOtcMem & \na & \valOverheadRevMem \\
\bottomrule
\end{tabular}
\end{table}

\begin{table}[ht]
\centering
\caption{Per-candidate feature-extraction microbenchmark. ``Loop'' is a non-vectorized Python loop that processes one candidate at a time and rebuilds the per-call indexer from scratch ($O(N)$ per call); ``vectorized'' is the batched path used by the training pipeline, with a single shared indexer. Means over $10^4$ randomly sampled candidates. The two vectorized columns sit in the same $\sim\!50\,\mu$s regime on both datasets; the loop columns differ because of the indexer rebuild, not because of a graph-size-driven slowdown of the deployed (vectorized) implementation.}
\label{tab:overhead_micro}
\small
\begin{tabular}{l cc}
\toprule
Dataset & Loop ($\mu$s/cand.) & Vectorized ($\mu$s/cand.) \\
\midrule
Bitcoin OTC          & \valOverheadOtcCand     & \valOverheadOtcCandBatch \\
\texttt{tgbl-review-v2} & \valOverheadFairReview  & \valOverheadRevCand \\
\bottomrule
\end{tabular}
\end{table}

The honest reading is: motif extraction adds a one-time preprocessing cost that is small relative to a single training epoch on TGB-scale streams, and is well within memory budgets for the experiments reported here. We do not claim a free lunch: per-epoch overhead is non-trivial on Bitcoin OTC (small graph, small batch), and the loop implementation does not scale; the deployed implementation is the vectorized one. The vectorized batch path is also the one used during training, so the vectorized column of Table~\ref{tab:overhead_micro} is the operationally relevant cost.

\subsection{Interpretability of the learned motif embedding}
\label{app:interpretability}

The learned linear embedding $M$ recovers the three-axis structure of Section~\ref{sec:characterization} as a property of the \emph{trained model} rather than of the raw data: Figure~\ref{fig:heatmaps} visualizes $M$ for Bitcoin Alpha, Bitcoin OTC, and MOOC, with trust networks concentrating mass on dyadic recency/reciprocity ($A_1$) and MOOC on star diversity and triadic flow ($A_2,A_3$).

\begin{figure}[ht]
  \centering
  \begin{minipage}[t]{0.32\linewidth}\centering
    \includegraphics[width=\linewidth]{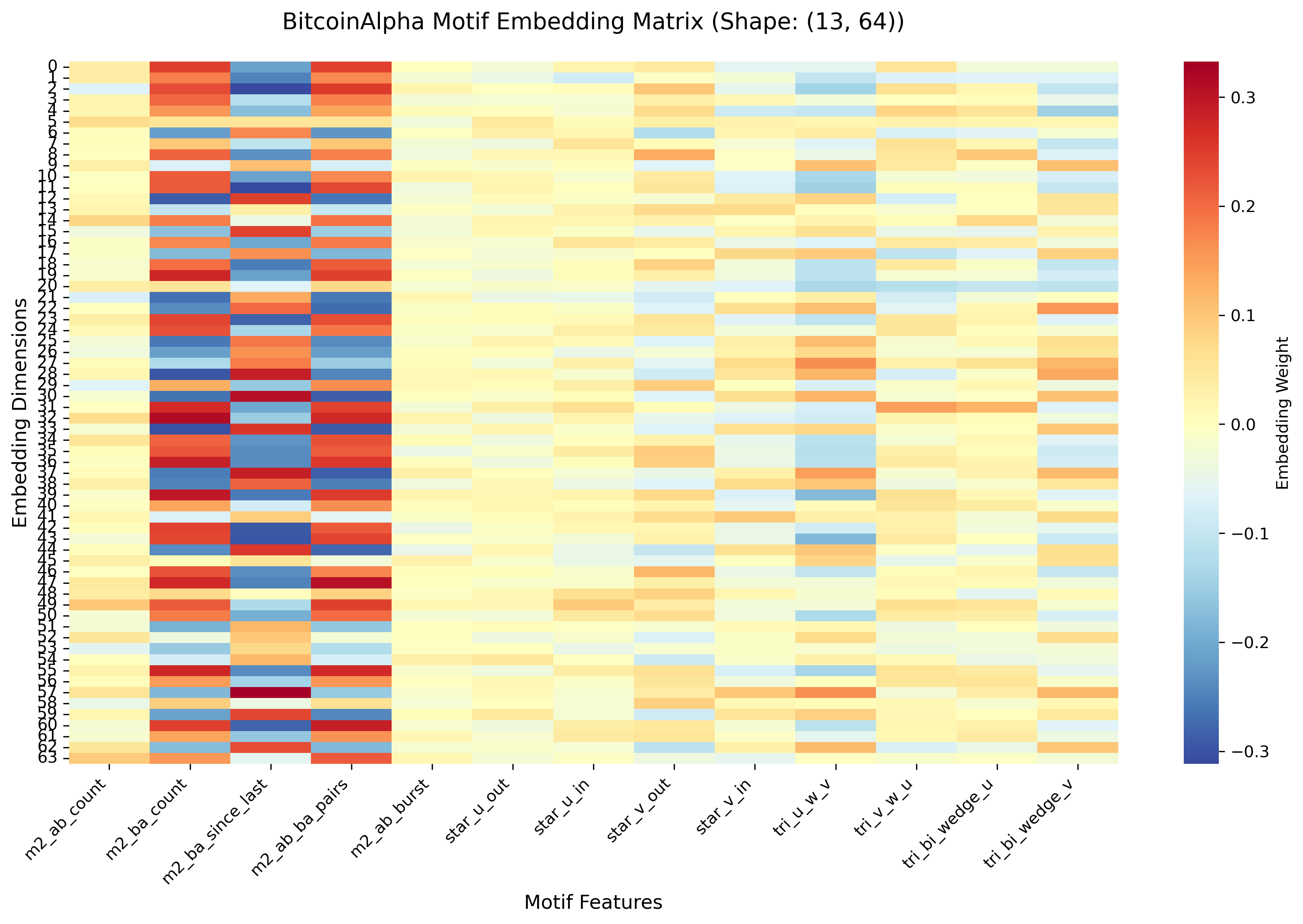}\\[-0.1em]{\footnotesize Bitcoin Alpha}
  \end{minipage}\hfill
  \begin{minipage}[t]{0.32\linewidth}\centering
    \includegraphics[width=\linewidth]{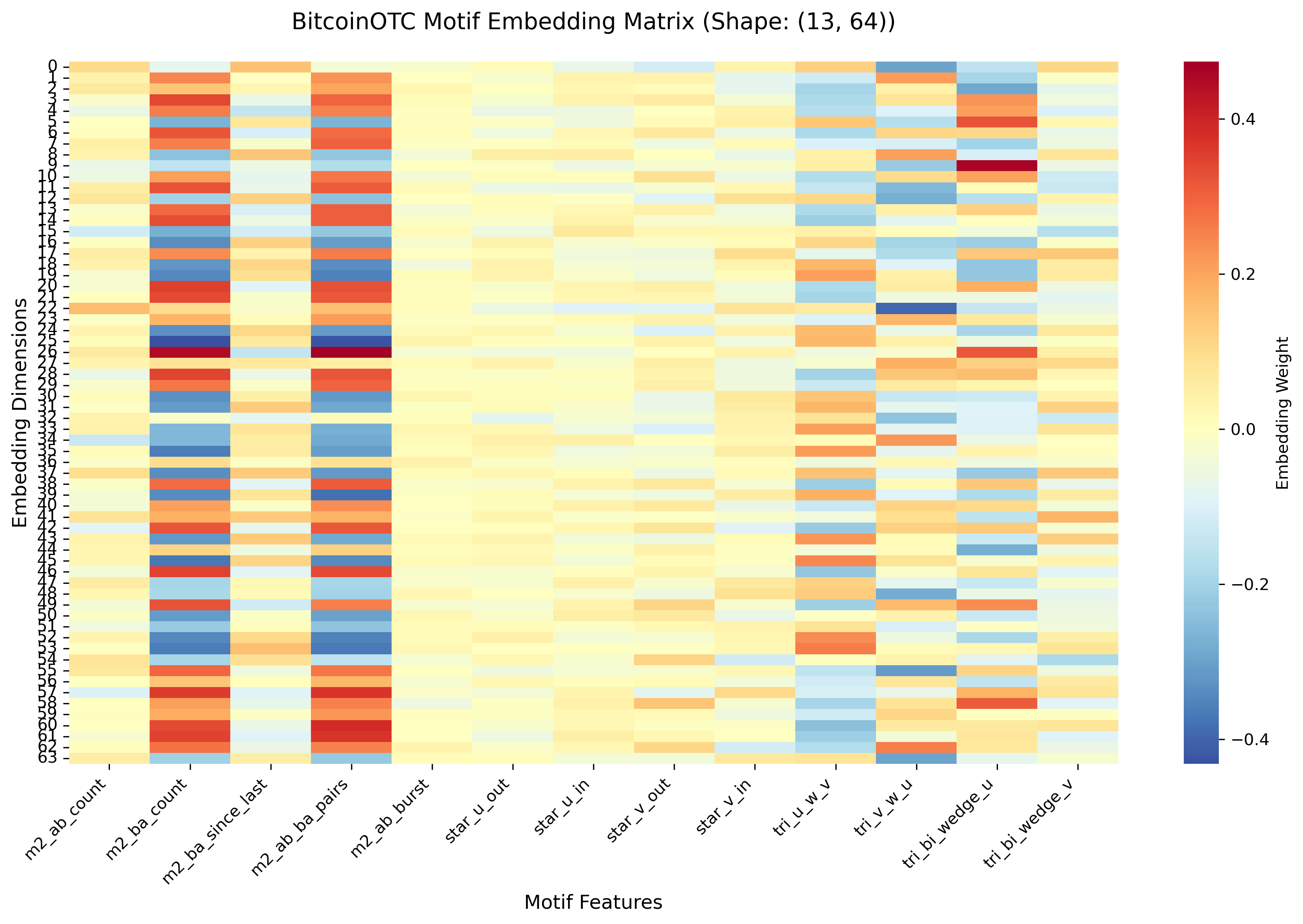}\\[-0.1em]{\footnotesize Bitcoin OTC}
  \end{minipage}\hfill
  \begin{minipage}[t]{0.32\linewidth}\centering
    \includegraphics[width=\linewidth]{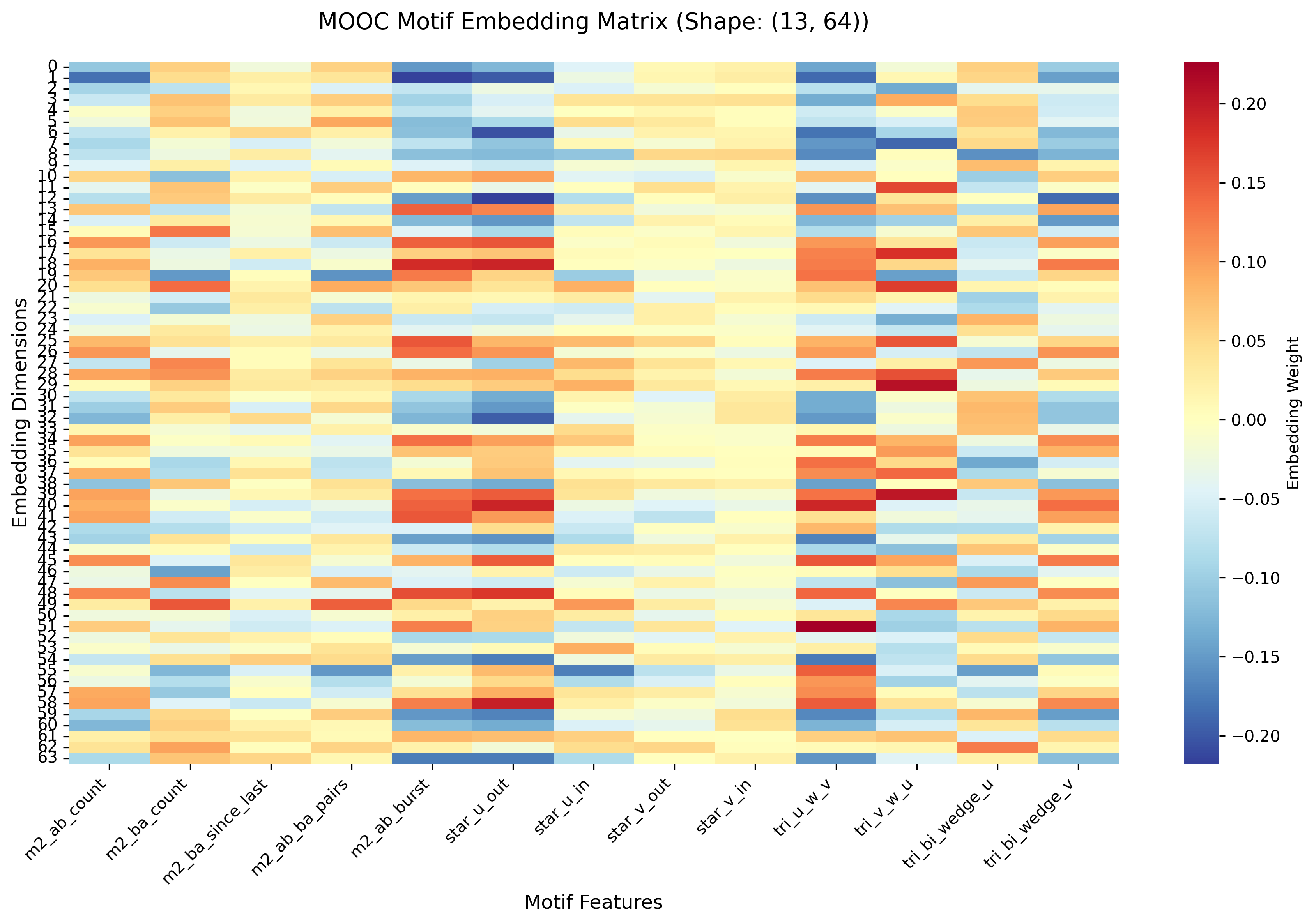}\\[-0.1em]{\footnotesize MOOC}
  \end{minipage}
  \caption{Learned linear motif embedding $M$ for Bitcoin Alpha (left), Bitcoin OTC (middle), and MOOC (right). Trust networks concentrate mass on $A_1$; MOOC on $A_2$ and $A_3$, reproducing the axis structure of Section~\ref{sec:characterization}.}
  \label{fig:heatmaps}
\end{figure}

\section{PaySim stress test}
\label{app:paysim}

The PaySim fraud-detection stream is included in this paper as a stress test rather than as a headline empirical result. Two structural properties of the dataset put it outside the regime where motif augmentation is expected to help: (i) the fraud subgraph has negligible triadic flow, so the four $A_3$ coordinates of $h$ are dataset-level near-zero (Section~\ref{sec:characterization}, Figure~\ref{fig:signatures}); and (ii) per-seed variance for SAGE/GCN/GAT exceeds the augmentation effect, with standard deviations of $\pm 0.07$--$\pm 0.17$ over three seeds for the SAGE backbone alone. We retain the dataset in the paper because it stress-tests the framework on a regime where one would not expect a gain, and because the per-family ablation row (below) is consistent with that interpretation.

\paragraph{Edge-classification table.}
Table~\ref{tab:paysim-edge-class} reports SAGE/GCN/GAT and the +Motifs (SAGE) augmentation row. TGN is omitted on PaySim due to a PyG \texttt{TGNMemory.train(False)} allocation failure when flushing the message store for $\sim$3M accounts on a single GPU; see Appendix~\ref{app:compute}.

\begin{table}[ht]
\centering
\caption{PaySim edge classification (fraud detection). The dataset is highly imbalanced, making PR-AUC particularly informative. Per-seed standard deviations are large; we read the +Motifs row as no-significant-gain rather than as a clear lift.}
\label{tab:paysim-edge-class}
\small
\setlength{\tabcolsep}{5pt}
\renewcommand{\arraystretch}{1.1}
\begin{tabular}{lcc}
\toprule
\textbf{Model} & \textbf{PR-AUC} & \textbf{ROC-AUC} \\
\midrule
SAGE             & \valPaySimSAGEPR & \valPaySimSAGEROC \\
GCN              & \valPaySimGCNPR  & \valPaySimGCNROC  \\
GAT              & \valPaySimGATPR  & \valPaySimGATROC  \\
\rowcolor{gray!12}
\textbf{+ Motifs (SAGE)} & \textbf{\valPaySimMotifsSAGEPR} & \textbf{\valPaySimMotifsSAGEROC} \\
\bottomrule
\end{tabular}
\end{table}

\paragraph{Per-family ablation.}
The PaySim row of the per-family ablation deferred from Section~\ref{ssec:ablations}:

\begin{table}[ht]
\centering
\caption{Per-family ablation on the SAGE backbone, PaySim row. Adding $A_3$ to any family is a functional no-op ($A_1{=}A_1{+}A_3$, $A_2{=}A_2{+}A_3$, $A_1{+}A_2{=}\text{All}$ to within numerical noise), consistent with the PaySim signature having near-zero $A_3$ mass.}
\label{tab:family_ablation_paysim}
\footnotesize
\setlength{\tabcolsep}{4pt}
\renewcommand{\arraystretch}{1.05}
\begin{tabular}{l ccc ccc c}
\toprule
 & $A_1$ & $A_2$ & $A_3$ & $A_1{+}A_2$ & $A_1{+}A_3$ & $A_2{+}A_3$ & All \\
\midrule
PaySim & \valFamPaySimAOne & \valFamPaySimATwo & \valFamPaySimAThree & \valFamPaySimAOneAtwo & \valFamPaySimAOneAthree & \valFamPaySimAtwoAthree & \textbf{\valFamPaySimAAll} \\
\bottomrule
\end{tabular}
\end{table}

\paragraph{Motif-baseline row.}
On PaySim the three motif baselines (global counts, static-only anchored, our temporal anchoring) overlap within one standard deviation across seeds (errors $\pm 0.07$--$\pm 0.17$ on PR-AUC), so the apparent ordering should not be read as a clear win. A constant per-candidate vector (the global-counts baseline) cannot itself improve candidate \emph{ranking} at a fixed encoder, so the apparent global-counts lift is most likely seed-level variance and head-MLP capacity effects rather than a true dataset-level signal; the dataset-level $A_3$ near-vanishing means temporal anchoring's main advantage is unlikely to materialize on PaySim in the first place.

\begin{table}[ht]
\centering
\caption{Motif-baseline row, PaySim, deferred from Section~\ref{ssec:controls}. PR-AUC mean$\pm$std across three seeds. Per-row standard deviations are larger than between-column gaps; do not read this as a clear win for any variant.}
\label{tab:motif_baselines_paysim}
\small
\begin{tabular}{l cccc}
\toprule
 & SAGE & + global counts & + static-anchored & + ours (temporal-anchored) \\
\midrule
PaySim & \valPaySimSAGEPR & \textbf{\valMotifBLPaySimGlobal} & \valMotifBLPaySimStatic & \valMotifBLPaySimTemporal \\
\bottomrule
\end{tabular}
\end{table}

\section{Compute resources}
\label{app:compute}

All experiments run on a single shared cluster with NVIDIA A100 (40\,GB) GPUs and AMD EPYC 7763 host CPUs. The full set of runs required for every row and figure in the main text fits within one GPU-week. Preliminary experiments (superseded tuning grids, abandoned architectures) consumed approximately the same amount of compute as the final runs.

\paragraph{Dropped runs.}
Two experiments were planned but could not be completed in the available compute budget:
\begin{itemize}[leftmargin=1.2em,nosep]
\item \emph{PaySim TGN edge classification} (Appendix~\ref{app:paysim}, Table~\ref{tab:paysim-edge-class}, three seeds). The PyG \texttt{TGNMemory.train(False)} call attempts a single allocation for the entire message store of $\sim$3M PaySim accounts, exceeding the 24\,GB GPU memory available. The bug fix that removed the previous \texttt{is\_fraud} \texttt{AttributeError} is unrelated; the failure is in the underlying library's eval-mode flush. We omit the TGN row from the PaySim table and rely on the SAGE/GCN/GAT baselines plus the +Motifs (SAGE) augmentation row.
\item \emph{tgbl-review-v2 soft/hard kernel sweep} (Fig.~\ref{fig:soft_vs_hard}, four settings $\times$ three seeds). The baseline TGN training plus per-snapshot evaluation alone consumed $\sim$13\,h per seed, and motifs training would have added a comparable amount; the full sweep exceeds two weeks of GPU time. We report the soft-vs-hard comparison on the three smaller datasets only.
\end{itemize}
Both omissions are clearly flagged in the relevant captions.

\section{Extended related work}
\label{app:extended_related}

This appendix complements Section~\ref{sec:related} of the main text.

\subsection{Classical network motifs}
Network motifs were introduced as statistically over-represented small subgraphs in biological and technological networks \citep{milo2002network,alon2007network}. \citet{benson2016higher} generalized the picture by showing that motif-induced adjacency operators expose higher-order community structure invisible to edge-level summaries. The temporal extension of \citet{paranjape2017motifs} is the immediate predecessor of our feature family; in contrast to that paper's aggregate enumeration, we use temporal motif \emph{counts per candidate edge} in a past-only window.

\subsection{Motif-based GNN architectures}
Several families of GNN architectures couple motifs to message passing or attention. MotifNet \citep{monti2018motifnet} builds one adjacency matrix per motif and convolves across them. Motif-based attention \citep{peng2018graph,lee2019graph} defines attention neighborhoods from motif templates. Motif graph neural networks \citep{chen2023motif} and motif-based graph attention \citep{sheng2024mgats} combine motif-induced masks with learned per-motif weights. MotifExplainer \citep{yu2022motifexplainer} uses motifs for post-hoc explanation. The common thread in these works is that motifs \emph{modify the encoder}. Our approach deliberately leaves the encoder untouched and supplies motif information as a linear embedding; this makes the augmentation model-agnostic and allows us to pair motifs with any of the temporal encoders listed below.

\subsection{Temporal graph neural networks}
Continuous-time TGNNs span three roughly delineated families: (i) recurrent and event-driven embeddings \citep{trivedi2018representation,kumar2019predicting}; (ii) memory-based frameworks that decouple event storage from propagation \citep{rossi2020temporal}; and (iii) attention/transformer-style models over historical neighborhoods \citep{xu2020inductive,wang2021inductive,yu2023towards,cong2023graphmixer}. Benchmark-tuned baselines on TGB \citep{gao2025hyperevent,yu2024temporalwalk} currently lead the leaderboard but remain within these architectural families. \citet{xiong2025survey} survey the design space for temporal link prediction. \citet{longa2023graph} provide a recent survey of temporal GNNs with an open-problems discussion that motivates the gain-quantifying framing of Section~\ref{sec:tgb}.

\subsection{Temporal expressivity}
\citet{souza2022provably} analyze the expressive power of temporal graph networks and provide a related but distinct WL-style bound. Our temporal-WL hierarchy (Definitions~\ref{def:t1wl} and \ref{def:tkwl}) is self-contained, modeled on \citet{morris2019weisfeiler}, and makes the connection between motif order and WL rung explicit (Theorem~\ref{thm:hierarchy}). Feature-augmented expressivity beyond $1$-WL in static settings was pioneered by \citet{bouritsas2022improving,frasca2022understanding}; our construction is the temporal counterpart, with an explicit temporal witness pair (Appendix~\ref{app:proof_separation}).

\subsection{Motif augmentation baselines}
To our knowledge no prior work evaluates motif-count-as-feature augmentation for temporal GNNs on TGB in a controlled, per-candidate, leakage-safe fashion. The two closest comparators are (i) global motif counts (a dataset-level feature replicated on every candidate) and (ii) static-anchored motif counts (the Paranjape--Benson--Leskovec approach with $\delta=\infty$). We evaluate against both in Section~\ref{ssec:controls}.

\section{Additional discussion}
\label{app:discussion}

\paragraph{Why edge anchoring matters.}
Global motif counts ($N$ counts total over the entire stream) or node-level motif profiles ($O(|V|)$ vectors) are two natural alternatives to edge-anchored counts. Both are strictly less informative for edge prediction: the former collapses to a dataset constant and can only move predictions through a fixed bias; the latter ignores the direction and asymmetry of the candidate $(u,v,t)$. Section~\ref{ssec:controls} and Appendix~\ref{app:extra_signatures} confirm this empirically.

\paragraph{Why past-only windows matter.}
A common failure mode in temporal prediction pipelines is to compute ``features of the neighborhood at time $t$'' that inadvertently include events at time $t$ itself (or, worse, events slightly after $t$, when the pipeline collates multiple queries into a batch). Our strict past-only definition $\mathcal{W}^{\mathrm{past}}_t(\Delta)=\{i:t-\Delta\le t_i<t\}$ rules both out by construction, and Appendix~\ref{app:protocol} describes the memory-update safeguards needed when pairing with stateful temporal encoders.

\paragraph{Which coordinates matter where.}
Figure~\ref{fig:heatmaps} together with Table~\ref{tab:leave_one_out} provide two independent views of feature importance--one from the learned embedding matrix and one from leave-one-out ablations. They agree to first order, supporting the claim that the $13$-coordinate basis is \emph{compact}--no single coordinate is uniformly redundant across datasets, with the largest leave-one-out impacts ($\approx-69\%$ on MOOC's \texttt{star\_u\_out}; $\approx-19\%$ on PaySim's \texttt{star\_v\_out}; $\approx+14\%$ on PaySim's \texttt{m2\_ba\_since\_last}) concentrated on different coordinates per dataset--and that the three-axis decomposition of Section~\ref{sec:characterization} is a property of \emph{the task}, not of \emph{the model}. We do not claim minimality in a formal sense; minimality would require a subset-search criterion across all datasets simultaneously, which we leave as future work.

\section{Licenses, datasheets, and asset provenance}
\label{app:licenses}

\begin{itemize}[leftmargin=1.2em,nosep]
\item \textbf{Bitcoin Alpha / Bitcoin OTC} \citep{kumar2016edge}. Publicly available, standard academic-use terms. No personal identifiers.
\item \textbf{MOOC} \citep{kumar2019predicting}. Publicly available via the Jodie release; standard academic-use terms.
\item \textbf{PaySim} \citep{lopez2016paysim}. Synthetic; CC-BY 4.0.
\item \textbf{TGB (tgbl-wiki-v2, tgbl-review-v2, tgbl-coin-v2)} \citep{huang2023temporal,gastinger2024tgb}. Publicly available under the benchmark's license; we use the official evaluator.
\item \textbf{Chartalist} \citep{shamsi2022chartalist}. Publicly available academic dataset used as an upstream source for the coin dataset.
\item \textbf{Amazon Reviews} \citep{ni2019justifying}. Publicly available under research-use terms.
\item \textbf{Synthetic temporal graphs.} Generated locally from NetworkX random graph models \citep{erdos1960evolution,barabasi1999emergence,watts1998collective,newman2001random}; released with the paper under CC-BY 4.0.
\item \textbf{TGN, TPNet, DyGFormer, TNCN, HyperEvent.} Code reused from the original authors' releases; licenses preserved.
\end{itemize}

We will release motif-feature extraction scripts and all training pipelines under an MIT license, including synthetic-dataset generators, model-training scripts, and evaluation harnesses. No human-subject data is used.


\end{document}